\documentclass[10pt, conference, compsocconf]{IEEEtran}
\ifCLASSINFOpdf
\else
\fi
\usepackage{cite}
\usepackage{amsmath,amssymb,amsfonts}
\usepackage{amsthm}

\usepackage{algorithmicx}
\usepackage{graphicx}
\usepackage{textcomp}
\usepackage{xcolor}
\def\BibTeX{{\rm B\kern-.05em{\sc i\kern-.025em b}\kern-.08em T\kern-.1667em\lower.7ex\hbox{E}\kern-.125emX}}

\DeclareMathOperator*{\argmin}{arg\,min}

\usepackage{multirow}
\usepackage{subcaption}
\usepackage{algorithm}
\usepackage[noend]{algpseudocode}
\newtheorem{definition}{Definition}
\newtheorem{theorem}{Theorem}
\newtheorem{proposition}{Proposition}

\usepackage{mathtools}

\DeclarePairedDelimiter\floor{\lfloor}{\rfloor}

\begin{document}

\title{Adversarial Robustness of Similarity-Based\\ Link Prediction
}

\author{\IEEEauthorblockN{Kai Zhou}
\IEEEauthorblockA{\textit{Dept. of Computer Science and Engineering} \\
\textit{Washington University in St. Louis}\\
St. Louis, MO, USA \\
zhoukai@wustl.edu}
\and
\IEEEauthorblockN{Tomasz P. Michalak}
\IEEEauthorblockA{\textit{Institute of Informatics} \\
\textit{University of Warsaw}\\
Warsaw, Poland \\
tpm@mimuw.edu.pl}
\and
\IEEEauthorblockN{Yevgeniy Vorobeychik}
\IEEEauthorblockA{\textit{Dept. of Computer Science and Engineering} \\
	\textit{Washington University in St. Louis}\\
	St. Louis, MO, USA \\
yvorobeychik@wustl.edu}

}

\maketitle

\begin{abstract}
Link prediction is one of the fundamental problems in social network
analysis. A common set of techniques for link prediction rely on similarity
metrics which use the topology of the observed subnetwork to quantify
the likelihood of unobserved links.
Recently, similarity metrics for link prediction have been shown to be
vulnerable to attacks whereby observations about the network are
adversarially modified to hide target links.
We propose a novel approach for increasing robustness of
similarity-based link prediction by endowing the analyst with a restricted set of reliable queries which accurately measure the existence
of queried links.
The analyst aims to robustly predict a collection of possible
links by optimally allocating the reliable queries.
We formalize the analyst's problem as a Bayesian Stackelberg game in
which they first choose the reliable queries, followed by an adversary
who deletes a subset of links among the remaining (unreliable) queries
by the analyst.
The analyst in our model is uncertain about the particular target link the
adversary attempts to hide, whereas the adversary has full information
about the analyst and the network.
Focusing on similarity metrics using only local information, we show
that the problem is NP-Hard for both players, and devise two
principled and efficient approaches for solving it approximately.
Extensive experiments with real and synthetic networks demonstrate the
effectiveness of our approach.
\end{abstract}

\begin{IEEEkeywords}
	Social network analysis, link prediction, adversarial robustness, game theory
\end{IEEEkeywords}

%======================== Sections ===================

%----------- Introduction ---------------
\section{Introduction}

The availability of massive social network datasets has led to the widespread use of Social Network Analysis (SNA) tools. For instance, centrality measures are used to identify important individuals~\cite{freeman1978centrality}, while link prediction aims to uncover hidden or missing connections within the network~\cite{liben2007link}. At the high level, such SNA tools extract knowledge from the observed network data, and  the reliability of SNA critically relies on the veracity of these observed data.

However, network data collection (which subsequently grounds SNA) is not necessarily reliable.
Many modes of data collection are error-prone, including field surveys (which may suffer from imperfect participant recall) and digitally collected data (such as social media, in which ``friends'' may both exclude actual friends, and include people who have never met one another).
In addition to such non-adversarial noise in data collection, many SNA settings introduce incentives for individuals to deliberately subvert network analysis by tampering with the data collection process. For example, suppose that law enforcement is investigating a crime network and collecting information about this network from personal interviews.
The criminals may either themselves provide misleading information, or intimidate others to do so.

The systematic investigation of the latter problem of adversarial social network analysis has received some attention from the attacker's perspective in recent literature, with several approaches developed for defeating analysis techniques such as centrality analysis, community detection \cite{waniek2018hiding}, link prediction \cite{waniek2018attack, zhou2019attacking}, and node classification~\cite{zugner2018adversarial}.
However, there have been scarcely any approaches investigating how to make SNA robust to such attacks. We propose the first approach for robust similarity-based \emph{link prediction}---a core problem in social network analysis---in the presence of adversarial edge deletion.

We begin by modeling network data collection as follows. The analyst submits a set of node-pair queries to the \emph{environment} which returns \emph{edge} or \emph{non-edge} in response to each query, corresponding to the assessment whether a queried pair of nodes are connected.
This is an \emph{abstraction} of most data collection approaches, such as field interviews, phone call monitoring (in criminal cases), etc.
Based on the query results, we assume that the analyst will construct a subgraph and use a similarity metric to assess the likelihood of the existence of target edges that are not in the set of queries. In our setting, an attacker can modify the query results by changing \emph{edge} to \emph{non-edge} (equivalently, delete edges from the observed subgraph) for a limited subset of queries in order to hide a target link.
In our running criminal network example, the criminals would intimidate some of the interviewees to not disclose existing relationships known to these. To counter such attacks, we assume that the analyst can make a subset of their queries \emph{reliable}. For example, they may elicit a particular relationship through multiple interviews as well as other means (such as monitoring communications), significantly reducing the likelihood that an existing link is successfully hidden.

We model the interaction between the analyst and the attacker as a non-zero-sum Bayesian Stackelberg game in which the defender (analyst) first commits to a set of reliable queries, and the attacker chooses the set of links they will delete after observing the analyst's decision. 
The Bayesian nature of this game captures the uncertainty of the analyst about both the network itself, and the attacker's preference about which link they wish to hide. We are interested in finding the Strong Stackelberg Equilibrium of this game.  We show that for all local similarity metrics finding the attacker's best response in this game is NP-hard, with the difficulty arising from the nature of tie-breaking. We then propose two principled algorithms to approximate the defender's optimal commitment strategy in the Stackelberg game, one that casts an approximate version of the problem as an integer linear program, and another that identifies the set of critical links that the attacker is likely to delete and associates each link with an estimate of damage to the defender.

We conduct extensive experiments on several random graph models as well as real networks. We show, surprisingly, that the attack will not always harm the defender---by virtue of this being a non-zero-sum game, there are instances when the attack may actually increase the defender's utility. In the more typical cases where the attack will decrease the defender's utility, our proposed heuristic algorithm dramatically reduces the damage with only a small proportion of reliable queries. We also show that it is not always useful to have more reliable queries, if they are not carefully chosen: in particular, increasing the number of randomly chosen reliable queries may at times decrease the defender's utility.

%----------------------------------- Related Works ----------------------
\section{Related Work}
Link prediction, as formulated by Liben-Nowell and Kleinberg~\cite{liben2007link}, considers the problem of predicting hidden or missing links based on the network structure as well as other side information. 
An important line of work in link prediction focuses on the design of node pair similarity metrics~\cite{lu2011link,wang2015link}, with the view that a higher similarity score for a pair indicates a greater likelihood of the existence of a link.
Such similarity metrics are commonly classified as neighbor-based \cite{leicht2006vertex}, path-based \cite{katz1953new,lu2009similarity}, and random-walk-based \cite{fouss2007random}, depending on the information used.

Recently, much effort has been devoted to analyzing the vulnerability of social network analysis methods to adversarial manipulation. For example, Waniek et al. \cite{waniek2018hiding} study the vulnerability of centrality measures, which are indicators of the importance of individuals or groups within a network, to adversarial manipulation. Specifically for link prediction, authors in \cite{waniek2018attack} and \cite{zhou2019attacking} study how attackers can lower the similarities among target links perceived by a network analyst, thus evading link prediction, by modifying network topology. Zhang et al. \cite{zhang2016measuring} experimentally analyzed the robustness of several similarity metrics in link prediction to random noise. While prior efforts primarily focus on attacking link prediction algorithms, there is scarcely any systematic analysis of how to make link prediction approaches robust to attack, which is the subject of our work.

There are also related works in the realm of adversarial machine learning \cite{vorobeychik2018adversarial}, specifically on attacking representation learning approaches (based on, e.g., random walk \cite{perozzi2014deepwalk} or graph convolutional networks \cite{kipf2016semi}) over network data. However, such works mainly focus on the attacks of some separate learning tasks, such as node classification \cite{zugner2018adversarial,DBLP:journals/corr/abs-1902-08412}. To the best of our knowledge, the attacks explicitly designed for such learning based link prediction approaches are still unknown.

Game theory has been extensively used in security domains to model and analyze the behavior of defenders and attackers in adverse situations. In an important class of Stackelberg game models \cite{conitzer2006computing,korzhyk2011stackelberg} the defender first commits to a defense strategy, and the attacker then observes this commitment decision and optimally responds.
The typical solution concept for such games is the Strong Stackelberg Equilibrium \cite{korzhyk2011stackelberg}, in which the attacker breaks ties in the defender's favor.
Such a framework naturally fits in many real-world situations, demonstrating considerable success in both theory and practice \cite{sinha2018stackelberg}. 
Bayesian Stackelberg Games \cite{kiekintveld2011approximation,jain2011quality} are an extension of this framework in which the attacker has an informational advantage, providing a natural modeling approach for our setting.

%------------------------- Similarity based Link Prediction ----------------

\section{Similarity-based Link Prediction}
The key idea behind similarity-based link prediction is to assess the
likelihood of the existence of the link between a pair of nodes by
calculating how topologically similar these nodes are.  While a large variety of similarity metrics are used in different prediction systems, we focus on a category of metrics termed \emph{local metrics}. For these, the similarity $\mathsf{Sim}(u,v|\mathcal{G})$ between two nodes $u$ and $v$ given an observed subgraph $\mathcal{G}$ only depends on ``two-hop'' information about $u$ and $v$: their neighbors and neighbors' neighbors. 
As a result, local metrics are easy to compute and do not require global information about the network.

\begin{table}[tbh]
	\renewcommand{\arraystretch}{1.1}
	\centering
	\caption{List of representative similarity metrics. Specifically, $N(u,v)$ denotes the set of common neighbors of $u$ and $v$ and $d(u)$ denotes the degree of a node $u$.}
	\label{Table-Sim_Metrics}
	\begin{tabular}{l|l|l}
		\hline
		\hline
		&{\bf Local Metrics} & $\mathsf{Sim}(u,v)$\\
		\hline
		\multirow{5}{*}{SLM}& Adamic-Adar & $\sum_{w\in N(u,v)} \frac{1}{\log d(w)}$\\
		& Resource Allocation & $\sum_{w\in N(u,v)} \frac{1}{d(w)}$\\
		%\hline
		& Common Neighbours & $|N(u,v)|$\\
		& Jaccard & $\frac{|N(u,v)|}{d(u) + d(v) - |N(u,v)|}$\\
		& S\o rensen & $\frac{2|N(u,v)|}{d(u)+d(v)}$ \\
		\hline
		
		\multirow{4}{*}{ASLM}&Salton & $\frac{|N(u,v)|}{\sqrt{d(u)d(v)}}$\\
		& Hub Promoted & $\frac{|N(u,v)|}{\min(d(u),d(v))}$\\
		& Hub Depressed & $\frac{|N(u,v)|}{\max(d(u),d(v))}$\\
		& Leicht-Holme-Newman & $\frac{|N(u,v)|}{d(u)d(v)}$\\
		\hline
		\hline		
	\end{tabular}
\end{table}

We list the metrics considered in this paper in Table~\ref{Table-Sim_Metrics}, which cover most of the popular metrics for link prediction. 
We find it useful to classify them as Symmetric Local Metrics (SLM) and Asymmetric Local Metrics (ASLM), as defined below in Section~\ref{sec-attacker-response}.

%-------------------------- Robust Link Prediction Model --------------------

\section{Robust Link Prediction Model}
In our model, a network analyst (defender) faces an attacker whose goal is to hide certain links.

\subsection{Defense \& Attack Models}

While much of prior literature on link prediction is not explicit about how network data is acquired, modeling data acquisition is crucial for a principled approach to robust link prediction.
We model the data collection process as follows.
An analyst collects observations about the network via a set of queries $Q = \{(u_i,u_j)\}$ to the \emph{environment} (which is a proxy for actual data collection; for example, field interviews, communication monitoring, etc.), where each $(u_i,u_j)$ stands for a pair of nodes.
The environment responds with ``edge'' if $(u_i,u_j) \in Q$ is indeed an edge in the underlying network and with a ``non-edge'' otherwise.
Again, query response here is an abstraction; for example, it would correspond to an answer of a survey subject whether individuals $i$ and $j$ are friends.
Since data collection is costly, the number of such queries is limited. Hence, given the partial graph $\mathcal{G}_Q$ constructed with $Q$, the analyst employs link prediction algorithms to find whether there exist some other links in the network that have not been identified so far. Formally, the analyst wants to predict the existence of links among a set of node pairs, denoted by $H_D = \{(v_i,v_j)\}$, that do not appear in the observed network. Naturally, we assume $(v_i,v_j) \notin Q$, for any $(v_i,v_j) \in H_D$, since otherwise the existence of this edge is directly observed. We refer to $H_D$ as a \emph{target set} of the analyst.

We assume that the analyst uses similarity-based link prediction, i.e., she computes a similarity score $\mathsf{Sim}(v_i,v_j|\mathcal{G}_Q)$ between a pair of nodes $(v_i,v_j) \in H_D$ based on the observed subgraph $\mathcal{G}_Q$.  When clear from context, we often write the similarity metric simply as $\mathsf{Sim}(v_i,v_j)$. 
We assume that the analyst will predict that the link between $v_i$ and $v_j$ exists iff $\mathsf{Sim}(v_i,v_j) \ge \theta$, where $\theta$ is a pre-defined threshold.

The attacker aims to hide a target connection $H_A = (V_1,V_2)$ from the network analyst. Specifically, the attacker attempts to minimize the similarity score $\mathsf{Sim}(V_1,V_2|\mathcal{G}_Q)$ by modifying the results of a subset of the queries $Q$, following the model proposed by Zhou et al~\cite{zhou2019attacking}.
We restrict the attacker's ability to changing edges in $Q$ into non-edges, which is equivalent to deleting a subset of edges in $\mathcal{G}_Q$.
In practice, this can be achieved by making the existing links difficult to observe or measure (e.g., blocking communication channels, limiting communication, intimidating witnesses, etc.).  Since deleting links is typically costly (e.g., the connection between nodes comes from an actual need to communicate), we impose a constraint that the attacker can delete at most $k_A$ links.
Let the set of links removed by the attacker be denoted by $S_A \subseteq Q$.
The graph constructed by the analyst after the attack then becomes $\widehat{\mathcal{G}_Q} = \mathcal{G}_Q - S_A$.

We assume that the attacker knows the structure of the underlying graph $\mathcal{G}$, as well as the defender's target set $H_D$, modeling an informationally powerful attacker.\footnote{In fact, it suffices in our case for the attacker to only know local network structure for the target edge $H_A$, which is informationally quite plausible.}

To defend against the attack, we assume that the analyst can make a subset of queries $Q$ \emph{reliable} in the sense that the associated link information is accurately measured despite adversarial tampering.
For example, the analyst of a covert network can decide to devote sufficient resources (such as a background check, private investigation, etc.) to measure the connection between two nodes reliably. 
Let $S_D \subseteq Q$ denote the set of reliable queries chosen by the defender.
Clearly, reliably measuring links can be quite costly, and the analyst therefore faces a budget constraint that $|S_D| \leq k_D$ for an exogenously specified number of reliable queries $k_D$ she can make. For a combination of decisions $(S_D, S_A)$ by the analyst and the attacker about reliable queries and edges to remove, respectively, the observed subgraph becomes $\widehat{\mathcal{G}_Q} = \mathcal{G}_Q - S_A \setminus S_D$, where $S_A \setminus S_D = \{(u_i,u_j) \in S_A|(u_i,u_j) \notin S_D\}$.

\subsection{Bayesian Stackelberg Game Formulation}
At the high level, we have an adversarial situation, where the analyst attempts to distribute her budget of reliable queries within $Q$ so as to make link prediction vis-a-vis the target set of links $H_D$ robust against link removal attacks.
To formalize this, we model the interaction between the analyst and the attacker as a Bayesian Stackelberg game. In this game, the analyst first chooses the set $S_D$ of reliable queries.
The attacker then observes $S_D$, and chooses the set of edges $S_A$ to remove from $\mathcal{G}_Q$ (equivalently, the set of query answers to flip from ``edge'' to ``non-edge'').
Crucially, the analyst is uncertain about both $H_A$ (the attacker's target link to hide) and the true network $\mathcal{G}$ (as well as the derived subnetwork), which are both known to the attacker. We denote by $t = (H_A, \mathcal{G})$ the attacker's type or private information, upon which they can condition the choice of $S_A$.
Suppose that the analyst has a prior distribution $P$ over attacker's types $t$. Let the utilities of the attacker and the defender in the game given joint decisions $(S_D,S_A)$ and attacker type $t$ be $u_A(S_D,S_A;t)$ and $u_D(S_D,S_A;t)$, respectively.
The attacker can condition their strategy on both their type $t$ and the observation of the analyst's strategy $S_D$; we represent it as a function $g(S_D;t)$.

We are now ready to formally define the Strong Stackelberg equilibrium of the Bayesian Stackelberg game above.

\begin{definition}
	The strategy profile $\left\langle S_D^*, g^*\right\rangle $ forms a \emph{Strong Stackelberg Equilibrium (SSE)} of the Bayesian Stackelberg game if they satisfy the following:
	\begin{itemize}
		\item The defender plays a best response: $\forall \ S_D$,
		$$\mathbb{E}_{t\sim P}[u_D(S_D^*,g^*(S_D^*;t);t)] \geq \mathbb{E}_{t\sim P}[u_D(S_D,g^*(S_D;t);t)].$$
		\item The attacker plays a best response for each type $t$: 
		$$u_A(S_D,g^*(S_D;t);t) \geq u_A (S_D,g(S_D;t);t), \ \forall S_D,g.$$
		\item The attacker breaks ties optimally for the defender for each type $t$: $\forall S_D^*, g \in G(S_D^*;t)$, where $G(S_D^*;t)$ is the set of attacker's best responses to $S_D^*$ for attacker type $t$, $$u_D(S_D^*, g^*(S_D^*;t);t) \geq u_D(S_D^*,g(S_D^*;t);t).$$		
	\end{itemize}
\end{definition}

We now specialize the Bayesian Stackelberg game model to our problem by defining the utilities of both players.
First, recall that the goal of the attacker is to minimize the similarity of a target link $H_A$.
Consequently, $u_A(S_D,S_A;t) = -Sim(H_A)$.
To define the utility of the analyst, we begin by quantifying the quality of link prediction. Let  the loss of a prediction on a specific node pair $(v_i,v_j)$ be $l(\mathsf{Sim}(v_i,v_j)| \theta,y_{ij})$, where $\theta$ is a pre-defined threshold, $y_{ij} = 1$ indicates that $(v_i,v_j)$ is indeed an edge in the underlying graph and  $y_{ij} = -1$ otherwise. We do not assume a specific form of loss functions in our model as long as the loss is increasing in the similarity when $y_{ij} = -1$ and vise versa.
Then the total loss to the analyst is
\begin{align}
%\begin{split}
L(S_D,H_D|\theta,\mathcal{G},H_A) &= \sum_{(v_i,v_j) \in H_D} l(\mathsf{Sim}(v_i,v_j|\widehat{\mathcal{G}_Q}) | \theta,y_{ij}) \nonumber \\
&= -u_D(S_A,S_D),
%\end{split}
\end{align}  
where $\widehat{\mathcal{G}_Q} = \mathcal{G}_Q - g(S_D;(H_A,\mathcal{G}))\setminus S_D$, with $\mathcal{G}_Q$ a restriction of the graph $\mathcal{G}$ to the node pairs in $Q$; note the explicit dependence here of the attacker's strategy on the analyst's, as well as on their type.
Of course, in order to compute the analyst's loss function, they need to know two additional pieces, the graph $\mathcal{G}$ and $H_A$, which, as we may recall, jointly comprise the attacker's type $t$ distributed according to $P$.
Thus, the analyst's actual expected loss function is
\[
\bar{L}(S_D,H_D|\theta) = \mathbb{E}_{(\mathcal{G},H_A) \sim P}[L(S_D,H_D|\theta,\mathcal{G},H_A)].
\]
The analyst's utility is then the negative expected loss.

In our analysis, we use the sample average to approximate the expected loss. Specifically, we use a collection of $K$ samples $S = \{(\mathcal{G}^i, H_A^i\}_{i=1}^K$, and use the average loss over these samples as the objective:
\begin{align}
\tilde{L}(S_D,H_D|\theta) =  \frac{1}{K}\sum_{(\mathcal{G}^i,H_A^i) \in S} L(S_D,H_D|\theta,\mathcal{G}^i,H_A^i).
\end{align}

To simplify notation, we henceforth write $g(S_D;\mathcal{G}_Q^i,H_A^i)$ as $S_A^i$, omitting the dependence on $S_D$ where it is evident from the context.

Computing the SSE of our Bayesian Stackelberg game (with respect to the approximate loss function above) then amounts to solving the following bi-level optimization problem (omitting the constant $1/K$ factor from the objective and abusing notation slightly to make the dependence of the loss function on the attack strategy $S_A^i$ explicit):
\begin{align}
\label{E:problem}
\begin{split}
\argmin_{S_D}\ & \sum_{(\mathcal{G}^i,H_A^i) \in S} L(S_D,H_D|\theta,\mathcal{G}_Q^i - S_A^{i*}\setminus S_D)\\  
\text{s.t.}\ &|S_D| \leq k_D,\\
& S_A^{i*} = \argmin_{S_A^i : |S_A^i| \le k_A}\  \mathsf{Sim}(H_A^i|\widehat{\mathcal{G}_Q^i}), \ \forall \ i.
\end{split}
\end{align}
We aim to find the analyst's optimal solution to Problem~\eqref{E:problem}.

%------------------   Complexity  --------------------
\section{Complexity}
Recent work showed that maximizing the attacker's utility is computationally efficient if we are not concerned about tie-breaking~\cite{zhou2019attacking}.
However, we now show that in our setting, where tie-breaking is a crucial aspect of the SSE equilibrium concept, even computing a best response for the attacker is hard.

\begin{theorem}\label{thm-hard}
	Given $S_D$, computing the attacker's best response that breaks ties in the defender's favor is NP-hard for all local similarity metrics listed in Table~\ref{Table-Sim_Metrics}.
\end{theorem}

\begin{proof}
	Given a target pair $(V_1,V_2)$, the attacker can identify a set of tuples $\{(V_1,w_i,V_2)\}_{i=1}^{|N(V_1,V_2)|}$. His goal is to choose a proper partition of the common neighbors $\{w_i\} = W_1 \cup W_2$: for $w_i \in W_1$ he will delete edge $(V_1,w_i)$ and for $w_j \in W_2$ he will delete $(V_2,w_j)$. We assumed that none of the critical edges belongs to $S_D$; otherwise, there is no need for the attacker to choose which one of $(V_1,w_i)$ and $(V_2,w_i)$ to delete.
	
	We use the Common Neighbors as an example and the proof can be extended to other local metrics straightforwardly. As proved in \cite{zhou2019attacking}, the optimal strategy for the attacker (without considering tie-breaking) is to delete either $(V_1,w_i)$ or $(w_i,V_2)$ for each $w_i$. That is, any partition is a best response in minimizing $\mathsf{Sim}(V_1,V_2)$. However, different partitions will result in different utilities of the defender.
	
	Let $H_D = \{(u_i,u_j)\}$. We focus on a special case of the problem where all target links in $H_D$ are non-edges and any node in $H_D$ is a common neighbor of $V_1$ and $V_2$. Then, minimizing the loss on $H_D$ is equivalent to minimizing the sum of similarities of all node pairs $(u_i,u_j)$ in $H_D$, denoted as $S_H = \sum_{(u_i,u_j) \in H_D} \mathsf{Sim}(u_i,u_j)$. Thus, for the attacker, finding the best response breaking ties in the defender's favor is equivalent to finding a partition that minimizes $S_H$. We consider the decision problem termed $P_A$: can the attacker find a partition such that $S_D = k_a$?
	
	We use the decision version of the \emph{maximum cut} (MaxCut) for the reduction, which is known to be NP-complete. Given a graph $\mathcal{G} = (U,E)$, a cut is a partition of the nodes $V$ into two disjoint subsets $S$ and $T$. The cut-set is then the set of edges where each edge has one end-node in $S$ and the other one in $T$. The size of the cut-set is the number of edges in the set. MaxCut is to decide whether there is a cut such that the size of the cut set if at most $k$.
	
	Given an instance of MaxCut $(\mathcal{G} = (U,E),k)$, we construct an instance of $P_A$ as follows. We first construct a graph $\mathcal{H}$ where each node $u_i \in U$ is also a node in $H$. We then add two nodes $V_1$ and $V_2$ and connect every $u_i$ with both $V_1$ and $V_2$. Let $(u_i,u_j)$ be a target link in $H_D$ if and only if $(u_i,u_j)$ is an edge in $\mathcal{G}$. Then the problem $P_A$ is to decide whether there is a partition of $U$ such that $S_D =k_a =  |E| - k$. We show that MaxCut and $P_A$ are equivalent.
	\begin{figure}[h]
		\begin{center}
			\includegraphics[width=2in,height = 2cm]{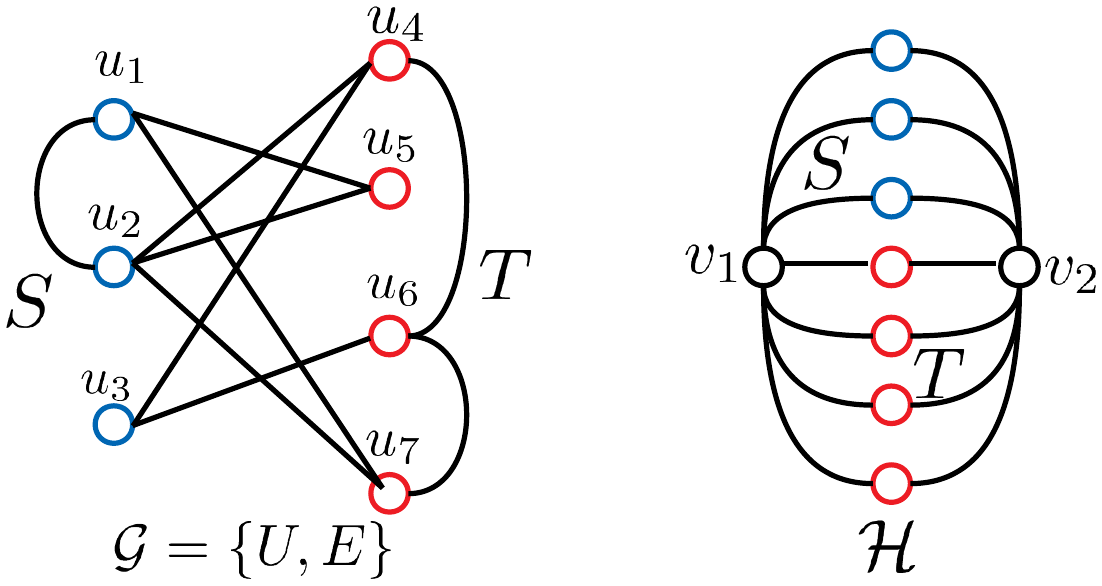}
			\caption{Example: construction of $\mathcal{H}$ from $\mathcal{G}$ in MaxCut.}
			\label{fig-np}
		\end{center}
	\end{figure}
	
	First, we show that if we can find a cut in graph $\mathcal{G}$ with size at least $k$, then we can find a partition of the nodes $U$ in graph $\mathcal{H}$ such that $S_D = |E| - k$.  Suppose the cut in graph $\mathcal{G}$ is $S$ and $T$ with size $k$ and the number of edges with both end-nodes in $S$ (respectively $T$) is $s$ (respectively $t$).  Then in graph $\mathcal{H}$, the attacker will delete all links $(u_i,V_1)$ for $u_i \in S$ and delete all links $(u_j,V_2)$ for $u_j \in T$. Now consider any target $(u_i,u_j)$. If $u_i$ and $u_j$ belong to different partitions (for example, $u_i\in S$ and $u_j \in T$), the similarity $\mathsf{Sim}(u_i,u_j) = 0$ as there is no common neighbors between $u_i$ and $u_j$ due to the deletion of links. If $u_i$ and $u_j$ belong to the same partition $S$, then $\mathsf{Sim}(u_i,u_j) = 1$ as they share only one common neighbor $V_2$. As there are $s$ edges with both end-nodes in $S$ in graph $\mathcal{G}$, we know there are $s$ target links in the partition $S$ in graph $\mathcal{H}$, each with similarity $1$. Similarly, we know there are $t$ target links in $T$ each with similarity $1$. Thus $S_D = s +t = |E| - k$. That is, we found a partition $S$ and $T$ of the nodes (which corresponds to a way of deleting edges) such that $S_D = |E| -k$. 
	
	Second, we show that if we find a partition of nodes in graph $\mathcal{H}$ such that $S_D \leq |E| -k$, then we found a cut of size at least $k$. Let the partition be $S$ and $T$ and the attacker delete all links $(u_i,V_1)$ for $u_i \in S$ and all links $(u_j,V_2)$ for $u_j \in T$. Consider a target link $(u_i,u_j)$ in $H_D$. If $u_i$ and $u_j$ belong to different partitions, then $\mathsf{Sim}(u_i,u_j) = 0$. If $u_i$ and $u_j$ belong to the same partition, we have $\mathsf{Sim}(u_i,u_j) = 1$. As $S_D = |E|- k$, we know there are $|E|- k$ such target links that the end-nodes belong to the same partition (either $S$ or $T$). Let $S$ and $T$ be a cut in graph $\mathcal{G}$. Since each target link in $H_D$ corresponds to an edge in $\mathcal{G}$, we have the number of edges with both end-nodes in the same partition is $|E| - k$. By definition, the size of the cut is $k$. That is, if we find a partition of nodes in graph $\mathcal{H}$ such that $S_D \leq |E| -k$, then we found a cut of size at least $k$.
\end{proof}
It is an immediate consequence that the problem of computing the SSE of our game is hard: simply let $k_D = 1$ (in this case, the problem is equivalent to computing the attacker's best response part of the SSE). Given that computing an optimal solution efficiently is out of the question, in what follows, we present principled approaches for computing an approximately optimal SSE strategy for the analyst.

%------------------------------  Solution ----------------

\section{Solution Approach}

At the high level, our approach makes an assumption that the damages caused by deleting links are approximately independent, which significantly simplifies the attack. Under this assumption, we show that the attacker's  best response that breaks ties  in the defender's (analyst's) favor can be found efficiently for all local metrics. Based on this, we then compute an approximately optimal defender strategy.

\subsection{Independent Damage Approximation}

In essence, the hardness of finding the attacker's best response that breaks ties in favor of the defender comes from the fact that the effects of deleting each link are inter-dependent. That is, the damage to the defender caused by deleting a particular link is determined by the states (either deleted or not) of other links. To make the problem tractable, we make the approximation that deleting a link will cause a damage which is independent of the states of other links. Our experiments subsequently demonstrate the effectiveness of this approach.
%which show the effectiveness of the algorithms we develop based on this approximation will serve to justify it.

Given the independence assumption, the results of Zhou et al.~\cite{zhou2019attacking} imply that for a sample $(\mathcal{G}^i , H_A^i)$, the attacker will only delete links connecting $V_1^i$ or $V_2^i$ with their common neighbors.
Specifically, the attacker will identify a subgraph $\mathcal{G}^i_A$ of $\mathcal{G}^i$ consisting of tuples $(V_1^i, w^i_j, V_2^i)$, where $w_j^i$  are common neighbors of  $V_1^i$ and $V_2^i$.  Denote the set of common neighbors by $W^i = \{w^i_j\}_{j=1}^{N^i}$, where $N^i = |N(V_1^i,V_2^i)|$.  Let $c_{jr}^i$ ($r \in \{1,2\}$) denote the damage caused by deleting link $(V_r^i, w_j^i)$, which is the change in the defender's loss
$L(S_D,H_D|\theta,\mathcal{G}^i)  - L(S_D,H_D|\theta,\mathcal{G}^i - \{(V_r^i, w_j^i)\}\setminus S_D)$.
Thus for each sample $(\mathcal{G}^i , H_A^i)$, the attacker can extract a weighed sub-graph $\mathcal{G}_A^i$, termed \emph{damage graph}, with each edge $(V_r^i,w^i_j)$  associated with a damage $c_{jr}^i$. Then under the independent damage assumption, the total damage $C^i$ caused by the attack is the summation of the damages $c_{jr}^i$ corresponding to the deleted links: $C^i = \sum_{(V_r^i,w^i_j)\in S_A^i} c_{jr}^i$.

We note that each individual damage $c_{jr}^i$ could be zero, positive, or negative. That is, the attack could possibly decrease the defender's loss, which reflects the non-zero-sum nature of the game between the attacker and defender. This is because deleting links could either increase or decrease the similarity score of a node pair in the defender's target set, which further will increase or decrease the defender's loss depending on the state (edge or non-edge) of that link.

\subsection{Computing Attacker's Best Response}\label{sec-attacker-response}
We seek to compute the attacker's  best response that breaks ties in favor of the defender under the independent damage assumption. We model strong attackers by assuming that $k^i_A = |N(V_1^i, V_2^i)|$.

We begin by considering attacker's strategy without considering tie-breaking. For an arbitrary sample $(\mathcal{G}, H_A)$ and its corresponding damage graph $\mathcal{G}_A$, the results in \cite{zhou2019attacking} show that the attacker would not delete the two links connecting to the same common neighbor simultaneously. Instead, he will choose one of $(V_1,w_j)$ and $(V_2,w_j)$ to delete for each $w_j$. Now, consider a tuple $(V_1,w_j,V_2)$. If both $(V_1, w_j)$ and $(w_j,V_2)$ are ``protected'' (i.e., reliably queried) by the defender, the attacker cannot delete either edge. If one of $(V_1, w_j)$ and $(w_j,V_2)$ is protected, the attacker will delete the other unprotected edge. Thus, the only non-trivial attacker decision is to select which one of $(V_1, w_j)$ and $(w_j,V_2)$ to delete when both of them are unprotected. For convenience, we term such edges critical edges. We assume that there are $k_A'$  ($k_A' \leq |N(V_1,V_2)|$) pairs of critical edges, among which the attacker will delete $k_1$ critical edges connecting to $V_1$ and $k_2$ connecting to $V_2$ ($k_1 + k_2 = k_A'$).

Based on the results in \cite{zhou2019attacking}, we classify the local metrics into Symmetric and Asymmetric metrics depending on the attacker's strategy in deleting critical edges. Specifically, for symmetric metrics, any combinations of $k_1$ and $k_2$ will maximize the attacker's utility as long as $k_1 + k_2 = k_A'$. For asymmetric metrics, the optimal solution requires some fixed  $k_1^*$ and $k_2^*$ with $k_1^* + k_2^* = k_A'$. The values of $k_1^*$ and $k_2^*$ can be efficiently computed given the degrees of $V_1$ and $V_2$ and the number of their common neighbors \cite{zhou2019attacking}.

Now, consider the attacker's best response when breaking ties in favor of the defender. Intuitively, the attacker will choose the set  $S_A$ that simultaneously minimizes the similarity of the target link and minimizes the total damage $C$.
The next two results characterize such a best response first for symmetric and then for asymmetric similarity metrics.

\begin{proposition}
	\label{prop-symmetric}
	For symmetric metrics, the attacker's best response that breaks ties in favor of the defender is: for each pair of links $(V_1,w_j)$ and $(V_2,w_j)$, if one of the links is protected, delete the other unprotected link and   if both links are unprotected, delete the link which is associated with a smaller damage.
\end{proposition}
\begin{proof}
	For each pair of links $(V_1,w_j)$ and $(V_2,w_j)$, if one of the links is protected, the attacker will delete the other unprotected link to minimize $\mathsf{Sim}(V_1,V_2)$.
	
	When both links are not protected, we use a binary variable $y_j$ to denote the attacker's decision regarding the tuple $(V_1,w_j,V_2)$, $j =1,2,\cdots,k_A'$. Specifically, $y_j = 1$ means that the attacker will delete edge $(V_1,w_j)$ and $y_j = 0$ means that the attacker will delete $(w_j,V_2)$. The attacker will minimize the total damage $C$, which can be written as:
	\begin{align} 
	\label{OPT-attacker}
	\argmin_{y}\ C = \sum_{j=1}^{k_A'} c_{j1} y_j + c_{j2} (1-y_j)
	% = \sum_{j=1}^{k_A'} (c_{j1} - c_{j2}) y_j + c_{j2}.
	\end{align}
	
	For symmetric metrics, the above optimization problem is unconstrained, as every combination of $k_1$ and $k_2$ will maximize the attacker's utility. Then it's straightforward to obtain that the optimal solution is to set $y_j = 1$ if and only if $c_{j1} - c_{j2} \leq 0$ for $j = 1,2,\cdots,k_A'$.
\end{proof}

\begin{proposition}
	\label{prop-asymmetric}
	For asymmetric metrics, the attacker's best response that breaks ties in favor of the defender is : i) for each pair of links $(V_1,w_j)$ and $(V_2,w_j)$, if one of the links is protected, delete the other unprotected links; ii) among all the unprotected link pairs, select $k_1^*$ common neighbors $w_j$ in ascending order of $(c_{j1} - c_{j2})$ and delete the corresponding links $(V_1,w_j)$, and for the remaining $w_j$, delete $(V_2,w_j)$.
\end{proposition}
\begin{proof}
	For asymmetric metrics, the attacker solves optimization problem (\ref{OPT-attacker}) with an extra constraint $\sum_{j=1}^{k_A'} y_j = k_1^*$. Rewrite the objective as $C = \sum_{j=1}^{k_A'} c_j y_j + B$, where $c_j = c_{j1} - c_{j2}$ and $B = \sum_{j=1}^{k_A'} c_{j2}$ is a constant. Clearly, greedily setting $y_j = 1$ in ascending order of $(c_{j1} - c_{j2})$ gives the optimal solution.
\end{proof}

\subsection{Computing an Approximately Optimal Strategy for the Analyst}

Based on the attacker's best response characterized above, we propose two algorithms to find the defender's strategy. The first one, termed \textit{IDOpt} (Independent Damage Optimization), formulates the defender's problem as a nonlinear integer program which can be linearized using standard techniques and whose solution yields the defender's optimal strategy for symmetric metrics under the independent damage approximation. The second one, termed IDRank (Independent Damage Ranking), ranks the \emph{importance} of each link based on the accumulated damages, avoiding solving the optimization problem and thereby allowing a significant improvement in scalability.

\paragraph{IDOpt}
For each sample, the defender is facing an attacker strategically deleting edges over the damage graph $\mathcal{G}^i_A$. The challenging part is that as the underlying graph $\mathcal{G}^i$ is sampled over the same node set according to some distribution, the damage graphs can have overlapping edges (although they are independent in the view of different types of attackers). This makes finding the defender's best response regarding a single sample meaningless. Instead, the defender needs to jointly consider all the samples.

We first show that for symmetric metrics, the defender's problem can be formulated as a nonlinear integer program with linear constraints. Specifically, we use a binary variable $x^i_{jr}$ ($r =  1,2$) to denote the defender's decision of protecting (reliably querying) edge $(w_j,V^i_r)$, where $x^i_{jr} = 1$ means the defender choose to protect  the edge and $x^i_{jr} = 0$ otherwise.

Consider the tuple $(V_1^i,w^i_j,V_2^i)$, the defender has four different options regarding protecting the two links $(V_1^i,w^i_j)$ and $(V_2^i, w^i_j)$. Specifically, when the defender chooses to protect neither links, the expected damage is $\min \{ c^i_{j1},c^i_{j2}\}$, based on the attacker's best response. Thus, the expected damage to the defender regarding tuple $(V_1^i,w^i_j,V_2^i)$ is 
\begin{align}
\label{eqn-cost}
c_j^i(x^i_{j1},x^i_{j2}) = &c^i_{j2}x^i_{j1}(1-x^i_{j2}) + c^i_{j1}(1-x^i_{j1})x^i_{j2} \nonumber \\
&+ \min\{c^i_{j1},c^i_{j2}\}(1-x^i_{j1})(1-x^i_{j2}).
\end{align}

Under the independent damage approximation, the total expected damage to the defender is 
\begin{align}
C (\mathbf{x}) =   \sum_{i=1}^{K}\sum_{j=1}^{N^i} c_j^i(x^i_{j1},x^i_{j2}),
\end{align}
where $N^i$ denotes the number of common neighbors of $V_1^i$ and $V_2^i$ and $\mathbf{x}$ denotes the defender's joint decision.  Then minimizing the defender's total loss is equivalent to minimizing the total expected damage over all samples. The defender solves the following integer programming problem:
\begin{align}\label{eqn-qp}
\min_{\mathbf{x}}\  C(\mathbf{x}), \quad \text{s.t.}\ \sum_{i}^{K}\sum_{j=1}^{N^i} (x^i_{j1} + x^i_{j2}) \leq k_D.
\end{align}
The nonlinear terms involve only pairwise products of binary decision variables.
Since each such term can be linearized using standard techniques, the optimization problem can be cast as an integer linear program. 
We note that the decision variables $x^i_{jk}$ many appear multiple times in the samples as the critical edges may overlap in the reduced sub-graphs. However, each $x^i_{jk}$ is counted once in the above constraint. From the above analysis, we have the following proposition regarding the defender's optimal strategy.

\begin{proposition}
	Suppose the attacker's best response is as specified by Proposition~\ref{prop-symmetric}, then the solution to the integer program (\ref{eqn-qp}) gives the defender's optimal strategy.
\end{proposition}

\textit{IDOpt} is summarized in Alg.~\ref{alg-DEQP}. We also use \textit{IDOpt} as a heuristic for asymmetric metrics. Specifically, we solve Eqn~(\ref{eqn-qp}) to obtain the defender's strategy while in the simulated attacks, we let the attacker follow  the strategy as stated in Proposition~\ref{prop-asymmetric}.

\begin{algorithm}
	\caption{IDOpt}\label{alg-DEQP}
	\begin{algorithmic}[1]
		\For{$i = 1,2,\cdots,K$}
		\State generate sample $(\mathcal{G}^i, H_A^i)$
		\State add $c_j^i(x_{j1}^i,x_{j2}^i)$ to objective \Comment{defined in Enq~(\ref{eqn-cost})}
		\EndFor
		\State construct integer program \Comment{defined in Eqn~(\ref{eqn-qp})}
		\State solve integer program, output $\mathbf{x}$ \Comment{$\mathbf{x}$: reliable queries}
	\end{algorithmic}
\end{algorithm}

\paragraph{IDRank}

When the size of the graph or the number of samples becomes large, solving the integer program will not scale. For this purpose, we propose a second approach \textit{IDRank}.
At the high level, \textit{IDRank} is guided by considering the defender's optimal strategy when there is only one sample. Then \textit{IDRank} will assign an importance score to each critical edge and the edges are ranked by their accumulated importance scores. Finally, the edges with high scores are identified as those that the defender needs to protect.

Let $\mathcal{G}_A$ be the damage graph extracted from a sample. Let each common neighbor $w_j$ of $V_1$ and $V_2$ be associated with a weight $c_j = \min\{c_{j1},c_{j2}\}$. Let $N_p$ be the number of common neighbors whose weights are positive. We have the following proposition characterizing the defender's (analyst's) optimal strategy for a single attack sample.

\begin{proposition}
	\label{prop-rank}
	Suppose the attacker's best response is as specified by Proposition~\ref{prop-symmetric}.  The defender's optimal strategy is: 1) when $K_D \geq 2N_p$, the defender will protect edges $(w_j,V_1)$ and $(w_j,V_2)$ for $w_j$ whose weights are positive; 2) when $k_D < 2N_p$, the defender will select $\floor{\frac{k_D}{2}}$ common neighbors $w_j$  in descending order of their weights and protect all $(w_j,V_1)$ and $(w_j,V_2)$ for selected $w_j$. 
\end{proposition}
\begin{proof}
	When the attacker's best response is as stated in Proposition~\ref{prop-symmetric}, the total damage is $C = \sum_{j=1}^N c_j$, as the attacker will choose the edge with smaller damage to delete for each tuple $(V_1,w_j,V_2)$. 
	Consider a common neighbor $w_j$ whose weight $c_j$ is non-positive. Protecting one or both of $(V_1,w_j)$ and $(V_2,w_j)$ will not decrease $C$. For a common neighbor $w_j$ with positive weight, suppose $c_{j1} \leq c_{j2}$, i.e., $c_j = c_{j1} >0$. If the defender chooses to protect edge $(V_1,w_j)$, $C$ will increase by $c_{j2} - c_{j1}$, which is non-negative. If the defender chooses to protect edge $(V_2,w_j)$, $C$ will not change. If the defender chooses to protect both edges, $C$ will decrease by $c_{j1}$, which is positive. To minimize $C$, the defender's optimal strategy is to protect both edges. Thus, when $K_D \geq 2N_p$, the defender will protect edges $(w_j,V_1)$ and $(w_j,V_2)$ for $w_j$ whose weights are positive. When  $k_D < 2N_p$, the defender will select $\floor{\frac{k_D}{2}}$ common neighbors $w_j$  in descending order of their weights and protect all $(w_j,V_1)$ and $(w_j,V_2)$ for selected $w_j$. 
\end{proof}

Proposition~\ref{prop-rank} states that when there is a single sample, the defender will  protect the tuples $(V_1,w_j,V_2)$ (i.e., protect both edges $(V_1,w_j)$ and $(V_2,w_j)$) only if deleting both edges will cause positive damage. This gives some intuition when the defender jointly considers all the samples.

Based on this, \textit{IDRank} works as follows. For each damage graph $\mathcal{G}^i_A$, the defender will identify those common neighbors whose weights (i.e., $c^i_j = \min \{c^i_{j1},c^i_{j2}\}$) are positive and assign an importance score $c^i_j$ to both edges $(w^i_j, V^i_1)$ and $(w^i_j,V^i_2)$. After processing all the damage graphs, the defender obtains a list of edges ranked by their importance scores, from which the defender can pick the top $k_D$ edges to protect. \textit{IDRank} is summarized in Alg.~\ref{alg-DERank}.

\begin{algorithm}
	\caption{IDRank}\label{alg-DERank}
	\begin{algorithmic}[1]
		\State initialize a weight vector $\mathbf{w}$ for all possible edges $(V^i_r,w^i_j)$
		\For{$i = 1,2,\cdots,K$}
		\State generate sample $(\mathcal{G}^i, H_A^i)$
		\For{$j = 1,2,\cdots,N^i$}
		\If{$c_j^i > 0$}  \Comment{$c_j^i = \min\{c_{j1}^i,c_{j2}^i\}$}
		\State add $\mathbf{w}[(V^i_r,w^i_j)]$ by $c_j^i$ for $r = 1,2$
		\EndIf
		\EndFor
		\EndFor
		\State rank $\mathbf{w}$ in descending order
		\State output top $k_D$ edges \Comment{reliable queries}
	\end{algorithmic}
\end{algorithm}

%----------------------  experiments ------------
\section{Experiments}

\subsection{Datasets}
\paragraph{Synthetic} We consider two synthetic data sets generated
from two random graph models. Both of these two models will generate
graphs with a power-law distribution, which are used to model a large
range of social networks. The first data set, denoted as PA, are
generated from the Barabasi-Albert \cite{barabasi1999emergence} (or
Preferential Attachment) model. Each of the graphs has $n = 500$
nodes, and the average degree of each node is $\sim 10$.  The second dataset is generated from the configuration model \cite{catanzaro2005generation}, that is used to generate graphs with pre-defined degree distributions. The generated graphs have a degree distribution satisfying $P(k) \propto k^{-\gamma}$, where $\gamma$ is a controllable parameter. We set $n=500$ and $\gamma = 2.0$. In the generated data set (PLD), each node in the graphs has an average degree of around $5$. Both graph generators are implemented through SNAP \cite{leskovec2016snap}. 

\paragraph{Real networks} We consider two real social networks from \cite{rozemberczki2018gemsec}, denoted as TVShow and Gov, which represent Facebook pages of two categories ( TV show and government, respectively). The nodes represent pages, and the (undirected) edges represent the ``likes'' among them. TVShow has $3892$ nodes and $17262$ edges, and Gov has $7057$ nodes and $89,455$ edges, which is denser. We use the random-walk (with restart probability $c = 0.15$) sampling approach \cite{leskovec2006sampling} to generate random subgraphs each having $500$ nodes from TVShow and Gov. 

\subsection{Attack and Defense Methods}
We consider three attack and three defense methods. The first attack is proposed in \cite{zhou2019attacking}, which we term \textit{LinkDel}, and  it will delete links  according to  Proposition ~\ref{prop-symmetric} and Proposition~\ref{prop-asymmetric}.  The second attack, termed \textit{UnbiasedDel}, is based on a heuristic approach from \cite{waniek2018attack}. Specifically, \textit{UnbiasedDel}  will delete one of the two links $(V_1,w)$ and $(V_2,w)$ without bias in case they are both not protected, for each common neighbor $w$ of the attacker's target nodes. The third attack \textit{RandDel} is motivated by \cite{zhang2016measuring} where they measure the robustness of link prediction algorithms through random perturbation on the graphs. To simulate such random perturbation, \textit{RandDel} will delete each \emph{unprotected} link $(V_i, w)$ ($i=1,2$) with probability $p=0.5$.

The first two defense methods are \textit{IDOpt} and
\textit{IDRank}. We also consider a third defense as a baseline,
termed \textit{PPN} (Protect Potential Neighbors). \textit{PPN} will
protect a subset of links randomly sampled from a set of critical
links $E_c$, which are links between the defender's target node set
$V_D$ and the rest of the nodes in the network.

\paragraph*{Evaluation metric}
We evaluate the defense performance by simulating $2000$ independent attacks and measuring the changes in the accumulated loss. Specifically, let $L_0$ be the defender's accumulated loss when there is no attack. Let $L_A$ be the defender's loss under some attack $A$  when the defender cannot make any reliable queries. We use $L_D$ to denote the loss when the defender make reliable queries according to a certain defense strategy $D$. We are primarily interested in \emph{damage prevention}, which measures the amount of damage that can be prevented by defense. Formally, we define a \emph{damage prevention ratio} $\mathsf{DPR}^D_A$ as the percentage of damage that is prevented: 
\begin{align*}
\mathsf{DPR}^D_A = \frac{L_A - L_D}{L_A - L_0},
\end{align*}
where $D\in \{\textit{IDOpt},\textit{IDRank},\textit{PPN} \}$ represents a defense strategy and $A \in \{\textit{LinkDel},\textit{UnbiasDel},\textit{RandDel} \}$ denotes an attack method.  We note that a larger $\mathsf{DPR}$ means that the defense strategy is more effective, and $\mathsf{DPR}$ is not necessarily smaller than $1$, as theoretically the attacks may decrease the defender's loss.

\subsection{Vulnerable Cases}

In this section, we empirically evaluate the effect of attacks in
different scenarios, using
\textit{LinkDel} as a representative attack, and identify the vulnerable cases where the attack
has a relatively higher impact on the defender's loss.  Let $V_D$ (respectively, $V_A$) be the
union of end-nodes appeared in $H_D$ (respectively, $H_A$). We thus
classify the attacks by the distributions of $V_D$ and
$V_A$. Specifically, we consider two cases of $V_D$, termed as
\emph{clustering} and \emph{sparse}, where the nodes in $V_D$ are
randomly drawn from high degree nodes and from all nodes in the
network, respectively. We also consider two cases of $V_A$, termed
\emph{targeted}, where $V_A \subset V_D$, and \emph{sparse}, where $V_A$ are randomly drawn from all nodes in the network, respectively. As a result, we consider four different attack scenarios in combination: Random Sparse Attack (RSA), Random Clustering Attack (RCA), Targeted Sparse Attack (TSA), and Targeted Clustering Attack (TCA).

We first empirically evaluate the damage to the defender caused by each
attack. In our experiments, we keep $V_D$ fixed, which is implemented
by renumbering the node IDs in the sample graphs.  
We measure the accumulated losses from $1$k simulated attacks in each attack scenarios. We define the damage as the change in the accumulated loss before and after the attacks. In our experiments, we use the exponential loss function, defined as $l(e) = \exp(- y_e \beta (\mathsf{Sim}(e) - \theta))$, where $e$ represents a node pair in $H_D$, $\beta$ is a parameter, and $\theta$ is a pre-defined threshold. The relative quantities of the damages are similar for other standard loss functions. We present the results for four  representative metrics (Table~\ref{table-damage} ): CN, S\o rensen, RA, and Salton, that have their own features. Specifically, CN only considers the number of common neighbors; S\o rensen and Salton can be considered as generalized CN adjusted by node degrees; RA computes the degrees of the common neighbors. Among these, Salton is asymmetric and the others are symmetric.

\begin{table}[h]
	\caption{The damages (\%) caused by attacks in four different attacking scenarios on TVShow and PA.}
	\label{table-damage}
	\begin{center}
		\resizebox{5cm}{!}{
			\begin{tabular}{|c|c|c|c|c|}
				
				\hline\hline
				Metrics	& TCA  & RCA  & TSA  & RSA \\
				\hline
				\multicolumn{5}{|c|}{TVShow}\\
				\hline
				CN	        &$+\mathbf{23.524}$  &$+1.400$  & $+3.196$ & $+0.022$ \\
				\hline
				S\o rensen	&$+ \mathbf{11.274}$  &$+0.589$  & $+ 1.424$ & $- 0.002$ \\
				\hline
				RA	        &$+\mathbf{9.461}$  &$+0.101$  & $+3.371$ & $+0.005$ \\
				\hline
				Salton	    &$+\mathbf{22.632}$  &$+0.727$  & $+7.695$ & $-0.025$ \\
				\hline
				
				\multicolumn{5}{|c|}{PA}\\
				\hline
				CN	        &$+\mathbf{7.167}$  &$+1.072$  & $+0.394$ & $+0.138$ \\
				\hline
				S\o rensen	&$+ \mathbf{5.292}$  &$+0.589$  & $-0.302$ & $+ 0.061$ \\
				\hline
				RA	        &$+\mathbf{17.628}$  &$+2.408$  & $+0.952$ & $-0.046$ \\
				\hline
				Salton	    &$+\mathbf{8.911}$  &$+1.659$  & $+1.673$ & $+0.3945$ \\
				\hline
				\hline
				
			\end{tabular}
		}
	\end{center}
\end{table}

In general, the effects of random attacks (whether $V_D$ are clustered
or not) are almost negligible; while the targeted attack on clustered
$V_D$ causes significant damage for all metrics. This motivates us to
focus on TCA, where $V_D$ consists of relatively high degree nodes and
$V_A$ are sampled from $V_D$.  We note that TCA models are an important class of attacks in reality, where both the defender and the attacker are interested in important individuals (measured by their degrees) in the networks. While the defender wants to predict the mutual connections among them, an attacker aims to hide a particular relationship. 

The rest of this section focuses on evaluating the performance of our
proposed defense strategies in the TCA scenario. In our experiments,
we set $|V_D| = 10$ and consider all node pairs in $V_D$. Accordingly,
the set of critical edges $|E_c|$ has a size of $\sim 5000$. We sample
$V_D$ from high degree nodes such that on average there are 
$\sim 45\%-55\%$ edges in $H_D$ in the sample graphs. For the \textit{IDOpt} and \textit{IDRank} defense, we generate $4000$ samples to identify the important links to protect.

\subsection{Defense Performance under Attacks}

\paragraph{Defense under \textit{LinkDel} attack}
We present the $\mathsf{DPR}$ of the three defense methods, \textit{IDOpt}, \textit{IDRank}, and \textit{PPN}, under the \textit{LinkDel} attack on four datasets in Fig.~\ref{fig-PA} to Fig.~\ref{fig-Gov}.
In all the figures, the horizontal line $\mathsf{DPR} = 0$ represents attacking without defense and $\mathsf{DPR} = 1$ means the defender's loss drops to $L_0$ for which there is no attack.

\begin{figure}[htp!]
	\centering
	\begin{subfigure}[t]{0.235\textwidth}
		\centering
		\includegraphics[width=\textwidth,height = 2.5cm]{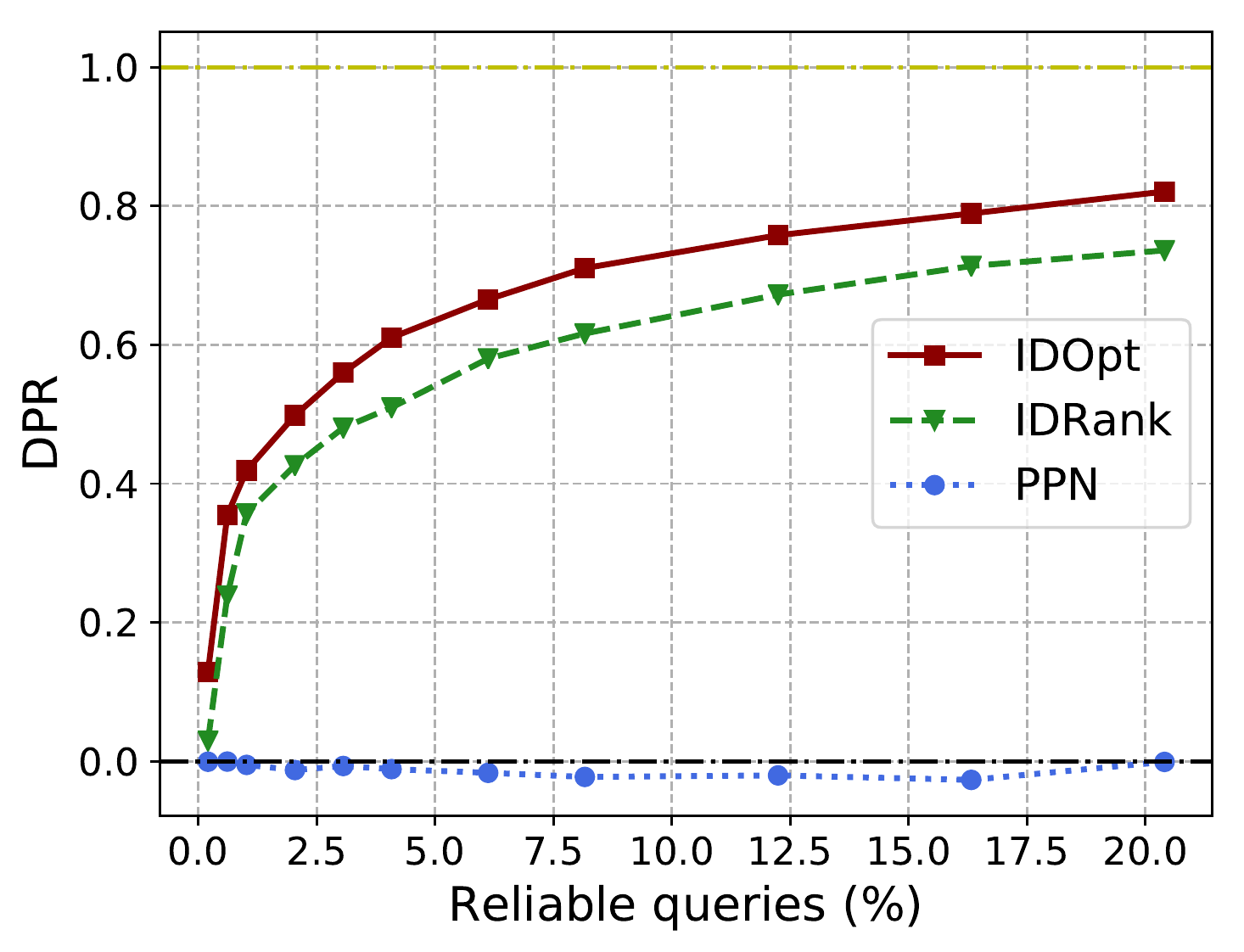}
		\caption{CN}
	\end{subfigure}%
	\hfill
	\begin{subfigure}[t]{0.235\textwidth}
		\centering
		\includegraphics[width=\textwidth,height = 2.5cm]{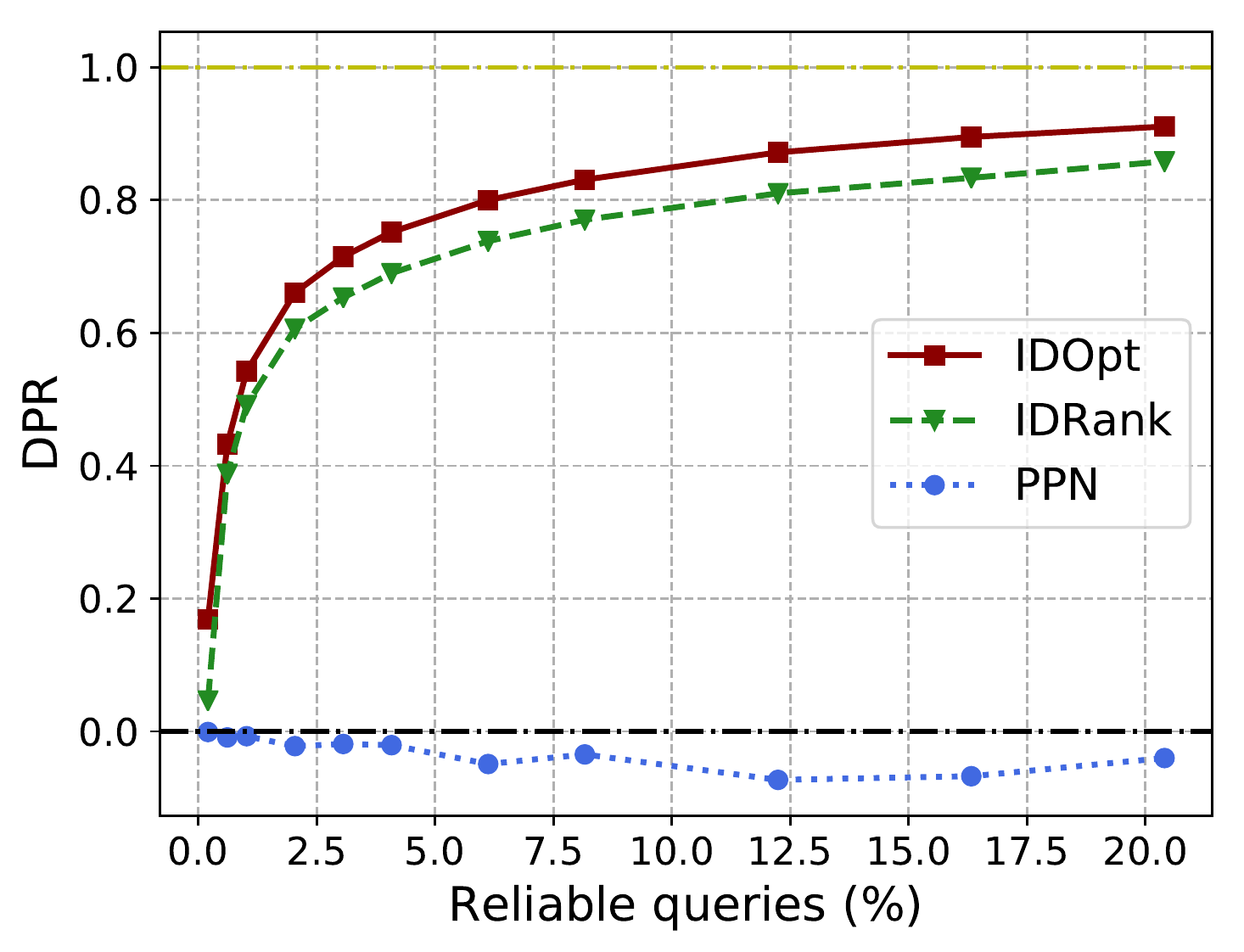}
		\caption{S\o rensen}
	\end{subfigure}	
	
	\begin{subfigure}[t]{0.235\textwidth}
		\centering
		\includegraphics[width=\textwidth,height = 2.5cm]{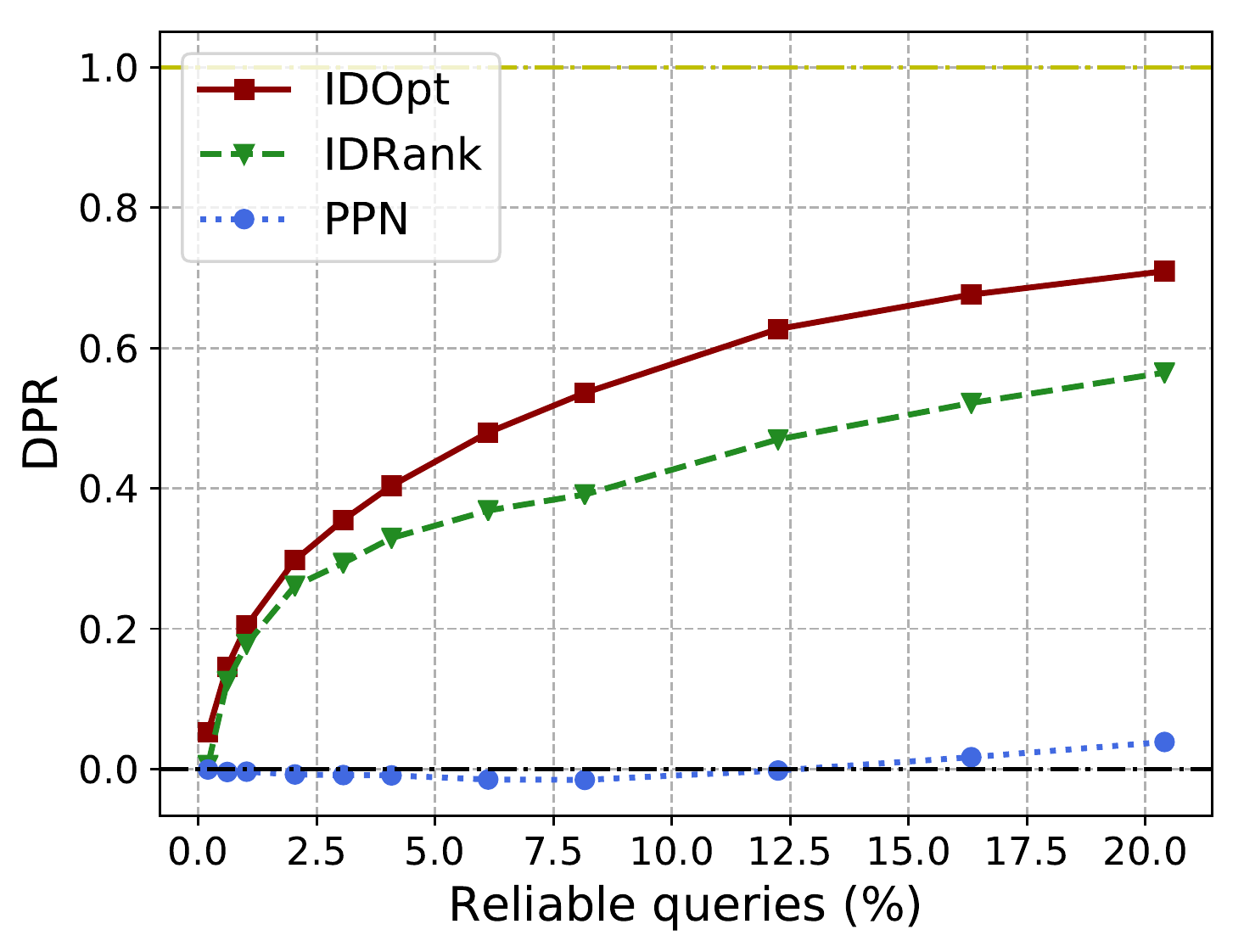}
		\caption{RA}
	\end{subfigure}
	\hfill
	\begin{subfigure}[t]{0.235\textwidth}
		\centering
		\includegraphics[width=\textwidth,height = 2.5cm]{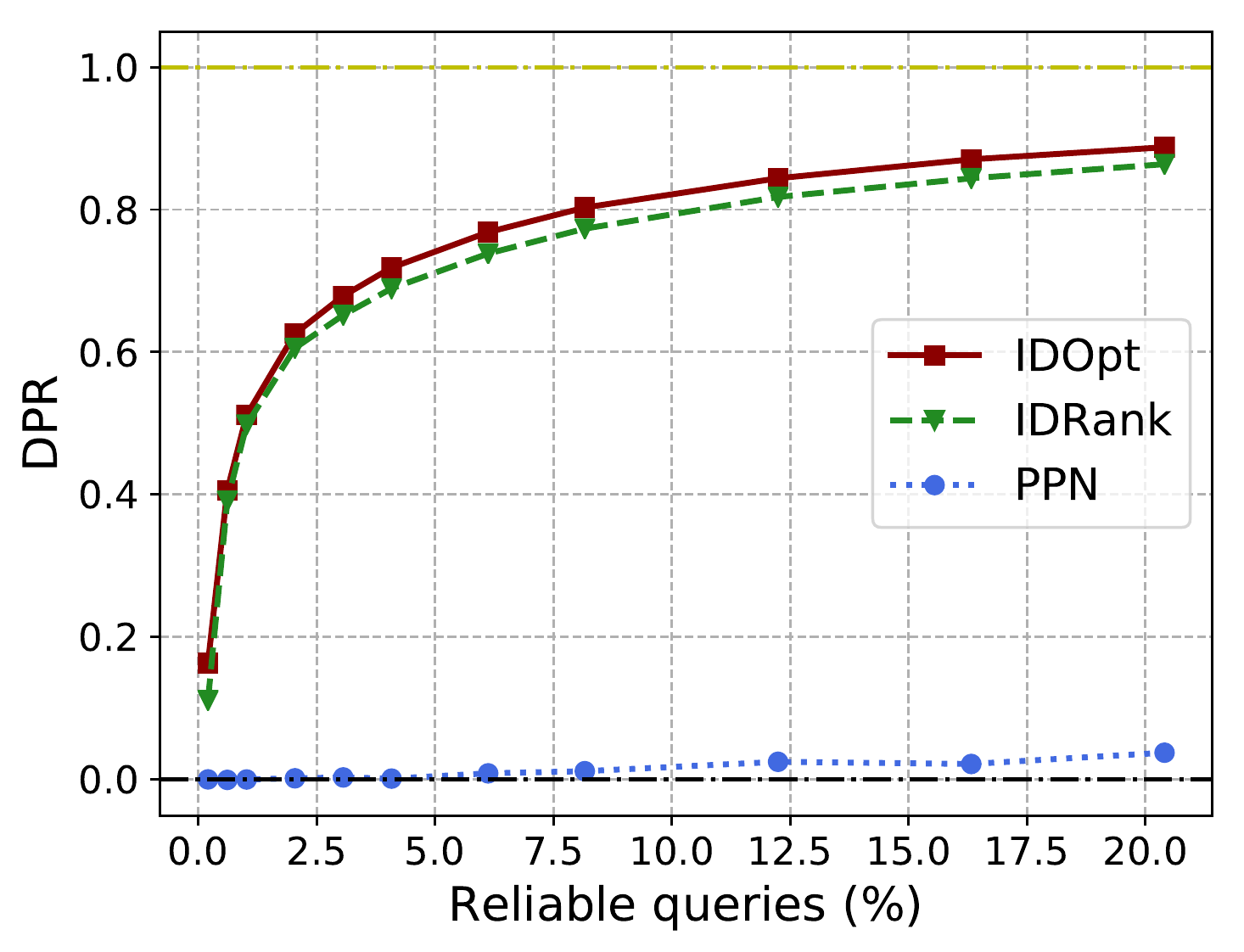}
		\caption{Salton}
	\end{subfigure}	
	\caption{$\mathsf{DPR}$ under \textit{LinkDel} attack on PA.}
	\label{fig-PA}
\end{figure}

\begin{figure}[htp!]
	\centering
	
	\begin{subfigure}[t]{0.235\textwidth}
		\centering
		\includegraphics[width=\textwidth,height = 2.5cm]{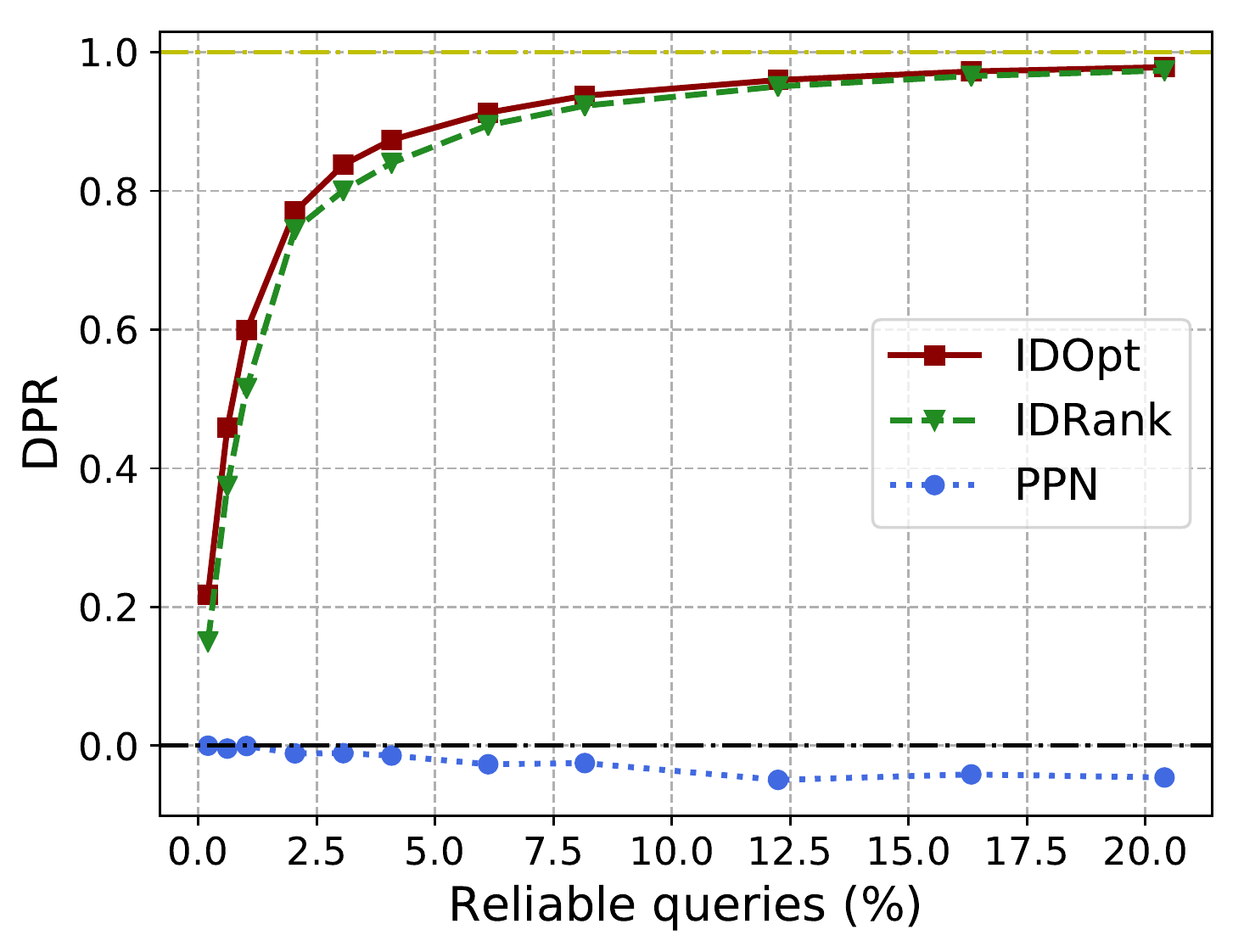}
		\caption{CN}
	\end{subfigure}%
	\hfill
	\begin{subfigure}[t]{0.235\textwidth}
		\centering
		\includegraphics[width=\textwidth,height = 2.5cm]{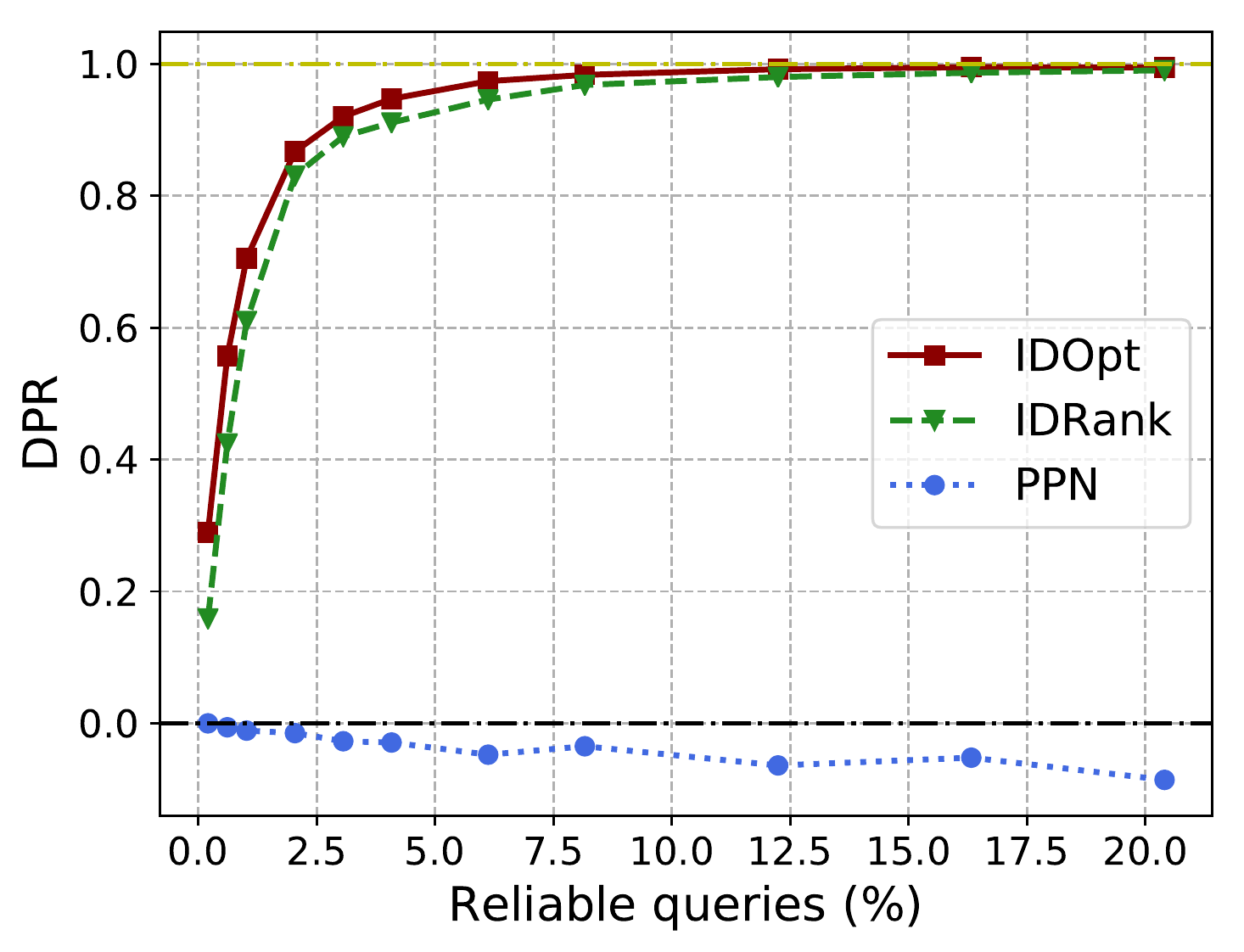}
		\caption{S\o rensen}
	\end{subfigure}	
	
	\begin{subfigure}[t]{0.235\textwidth}
		\centering
		\includegraphics[width=\textwidth,height = 2.5cm]{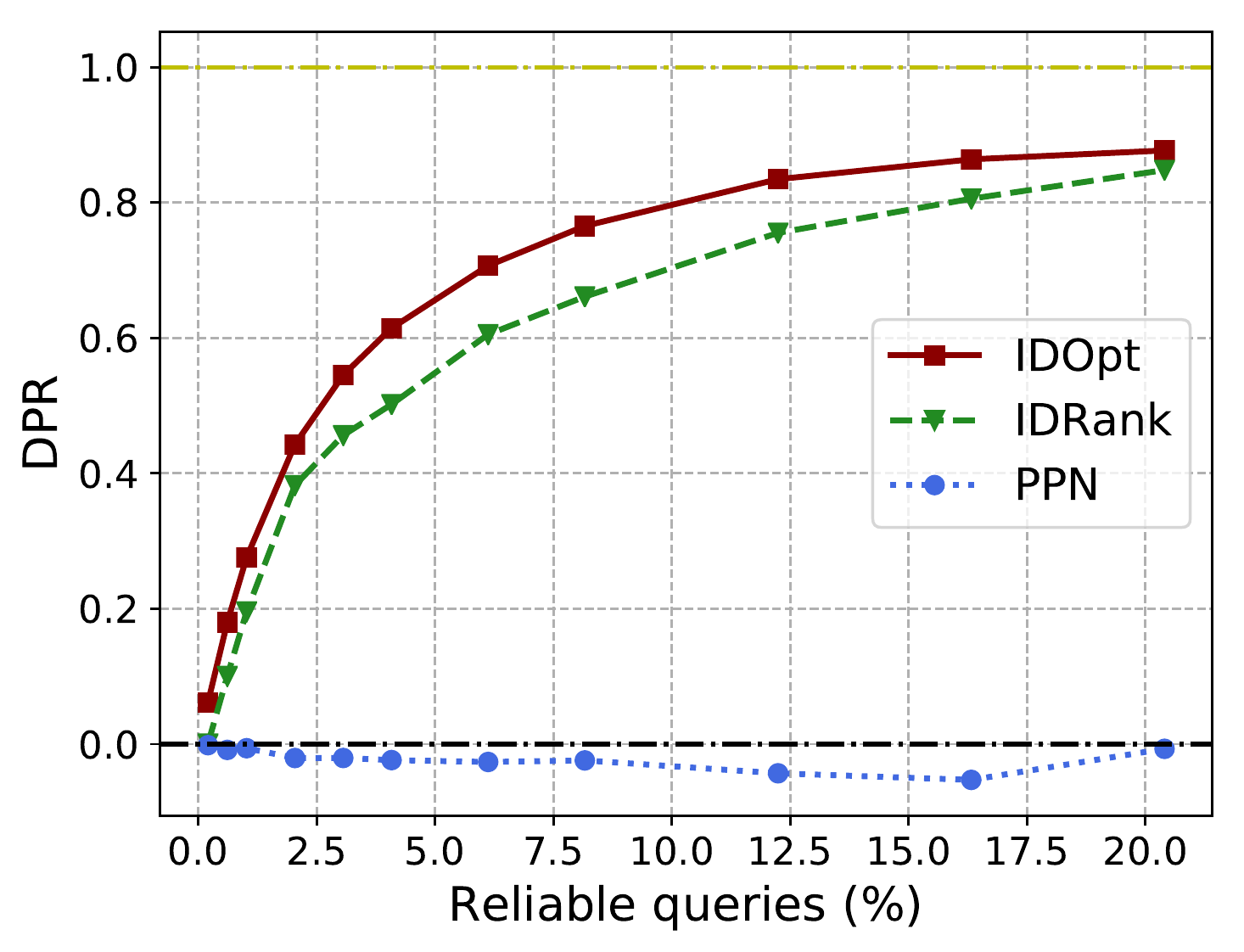}
		\caption{RA}
	\end{subfigure}
	\hfill
	\begin{subfigure}[t]{0.235\textwidth}
		\centering
		\includegraphics[width=\textwidth,height = 2.5cm]{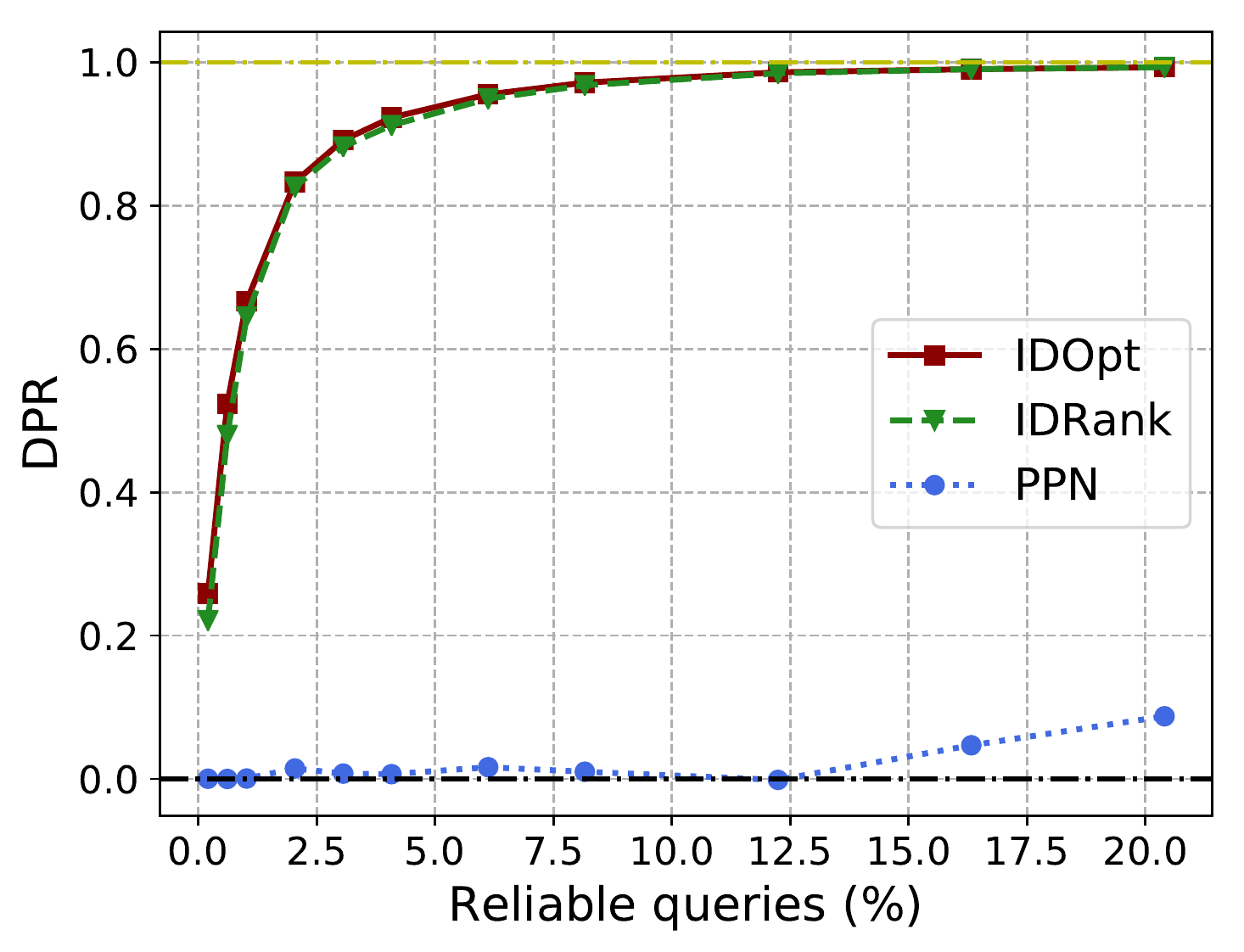}
		\caption{Salton}
	\end{subfigure}	
	\caption{$\mathsf{DPR}$ under \textit{LinkDel} attack on PLD.}
	\label{fig-PLD}
\end{figure}

\begin{figure}[htp!]
	\centering
	
	\begin{subfigure}[t]{0.235\textwidth}
		\centering
		\includegraphics[width=\textwidth,height = 2.5cm]{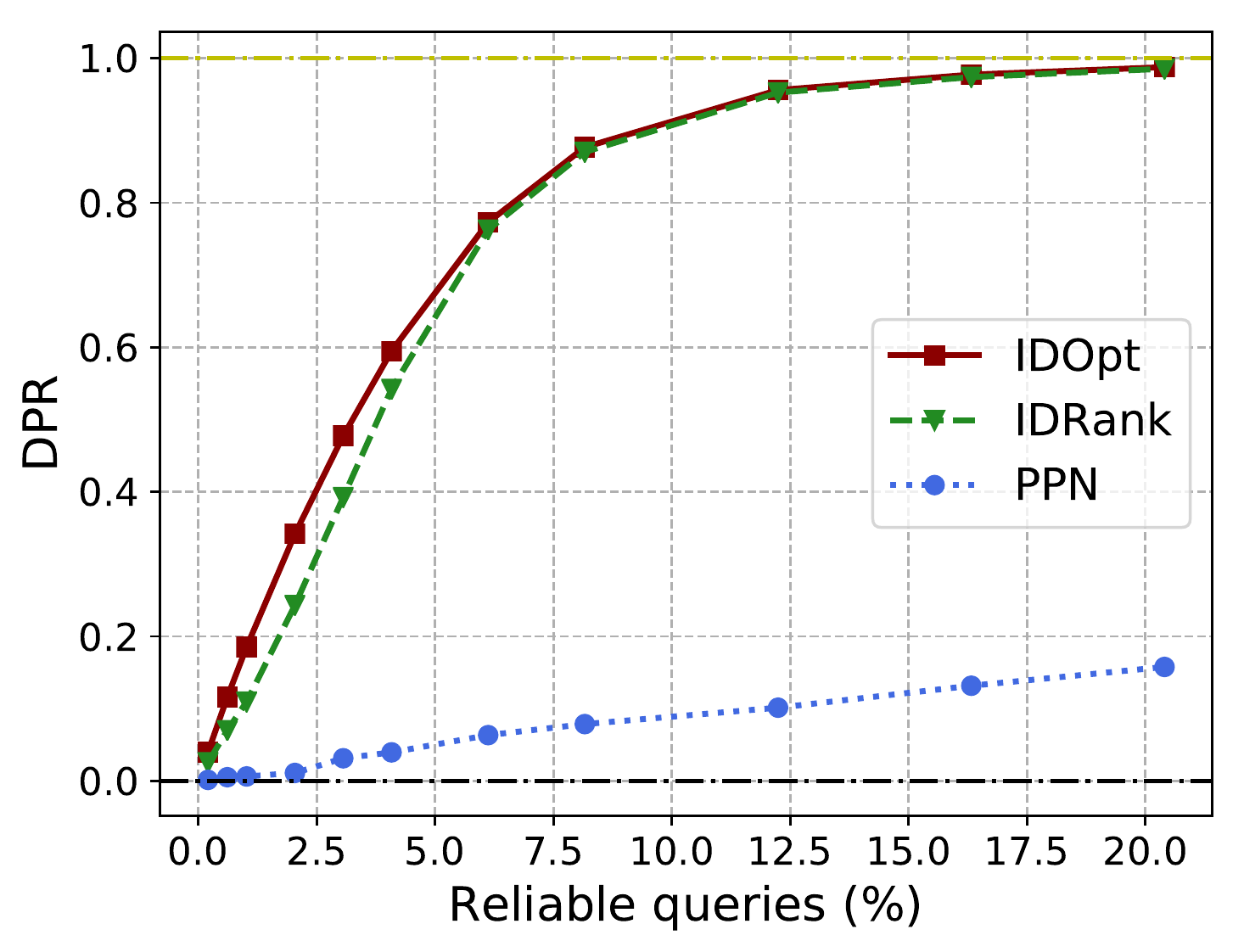}
		\caption{CN}
	\end{subfigure}%
	\hfill
	\begin{subfigure}[t]{0.235\textwidth}
		\centering
		\includegraphics[width=\textwidth,height = 2.5cm]{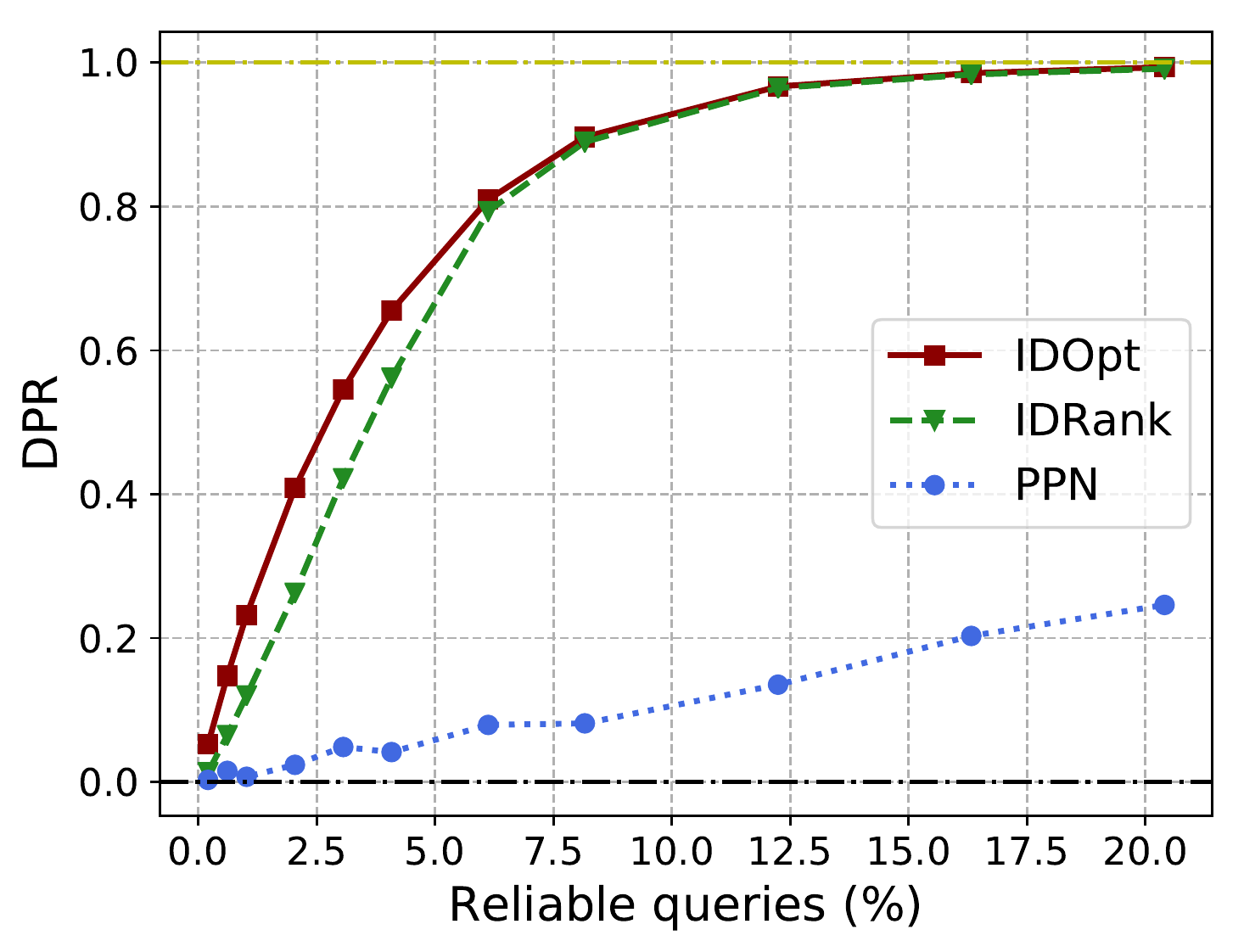}
		\caption{S\o rensen}
	\end{subfigure}	
	
	\begin{subfigure}[t]{0.235\textwidth}
		\centering
		\includegraphics[width=\textwidth,height = 2.5cm]{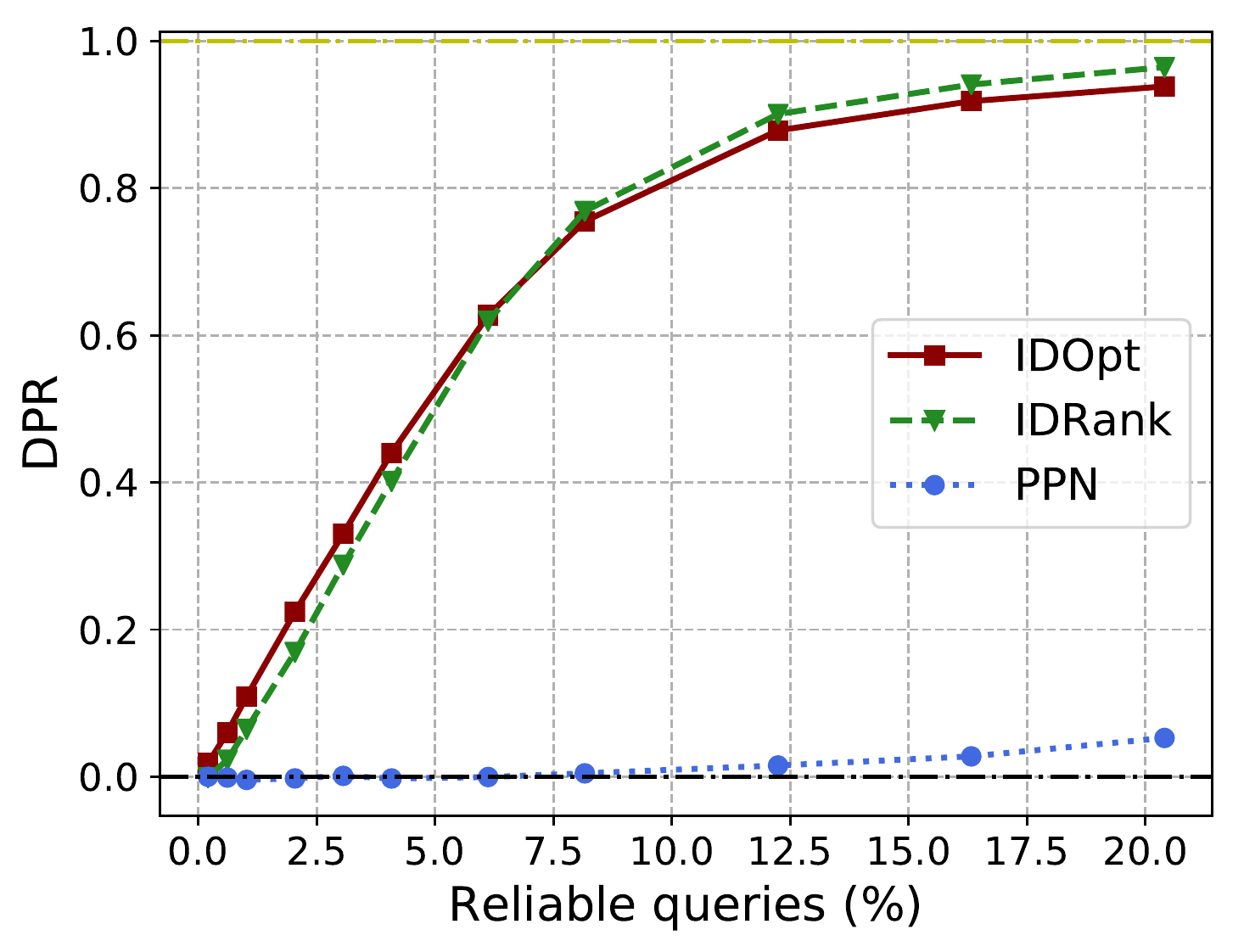}
		\caption{RA}
	\end{subfigure}
	\hfill
	\begin{subfigure}[t]{0.235\textwidth}
		\centering
		\includegraphics[width=\textwidth,height = 2.5cm]{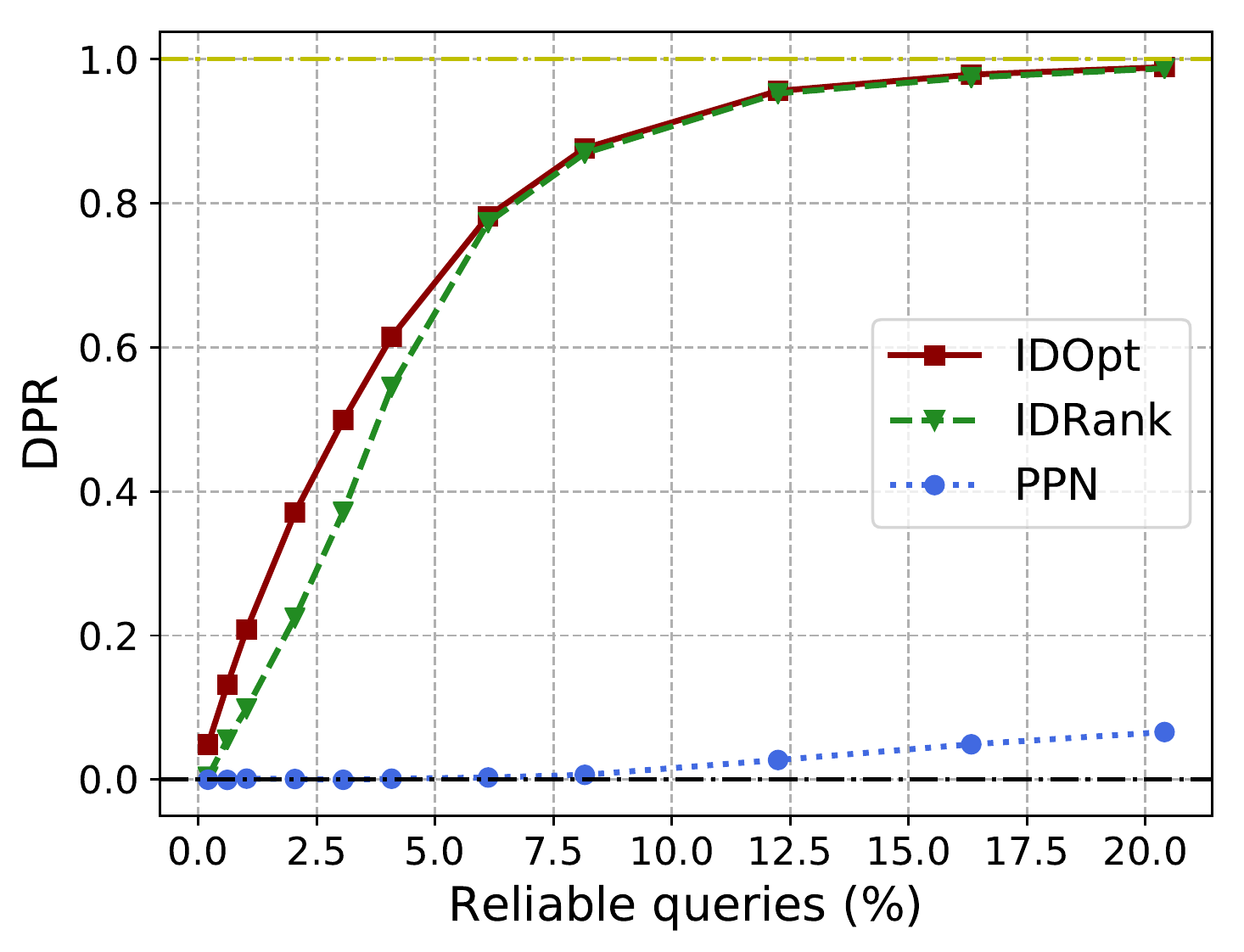}
		\caption{Salton}
	\end{subfigure}	
	\caption{$\mathsf{DPR}$ under \textit{LinkDel} attack on TVShow.}
	\label{fig-TVShow}
\end{figure}

\begin{figure}[htp!]
	\centering
	\begin{subfigure}[t]{0.235\textwidth}
		\centering
		\includegraphics[width=\textwidth,height = 2.5cm]{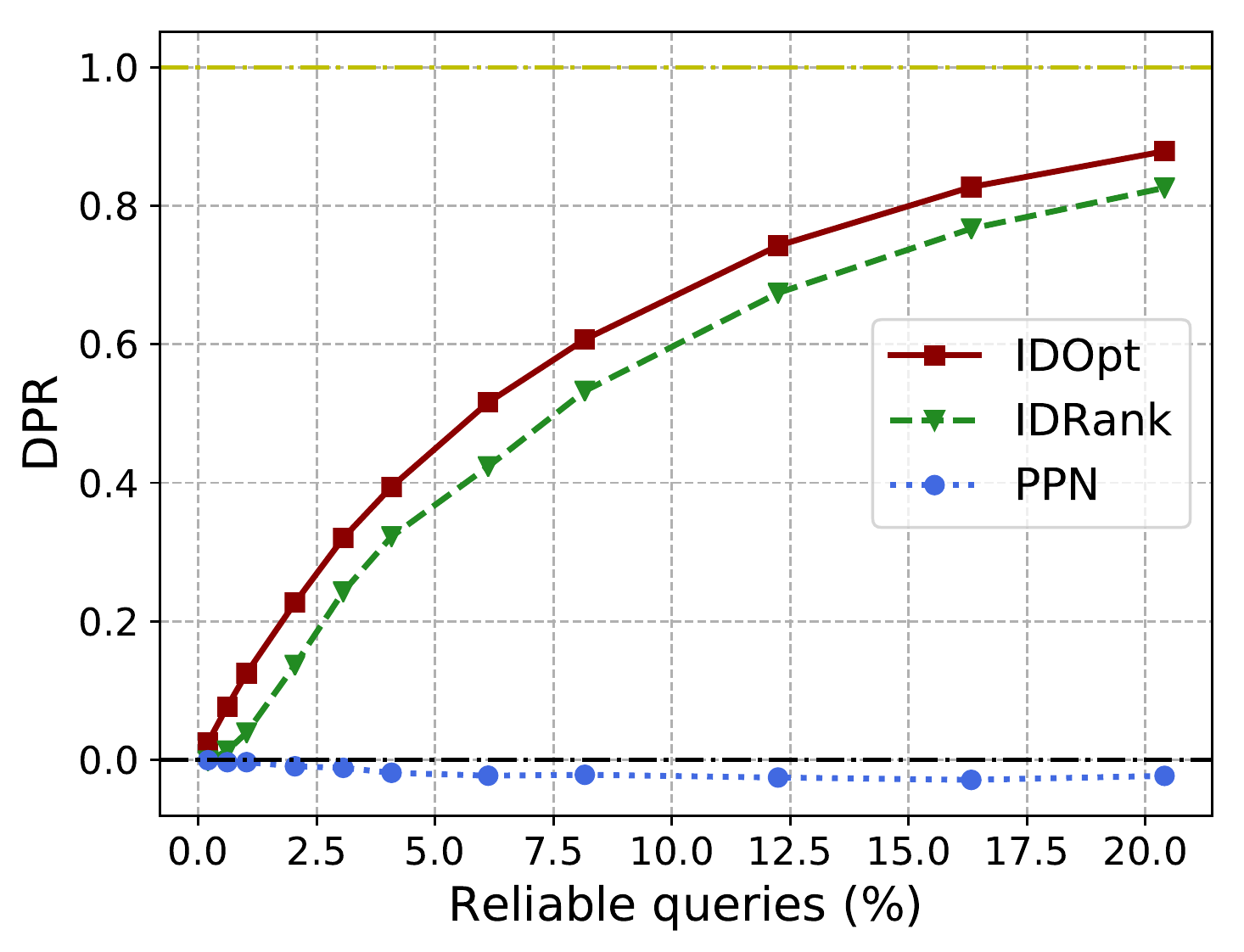}
		\caption{CN}
	\end{subfigure}
	\hfill
	\begin{subfigure}[t]{0.235\textwidth}
		\centering
		\includegraphics[width=\textwidth,height = 2.5cm]{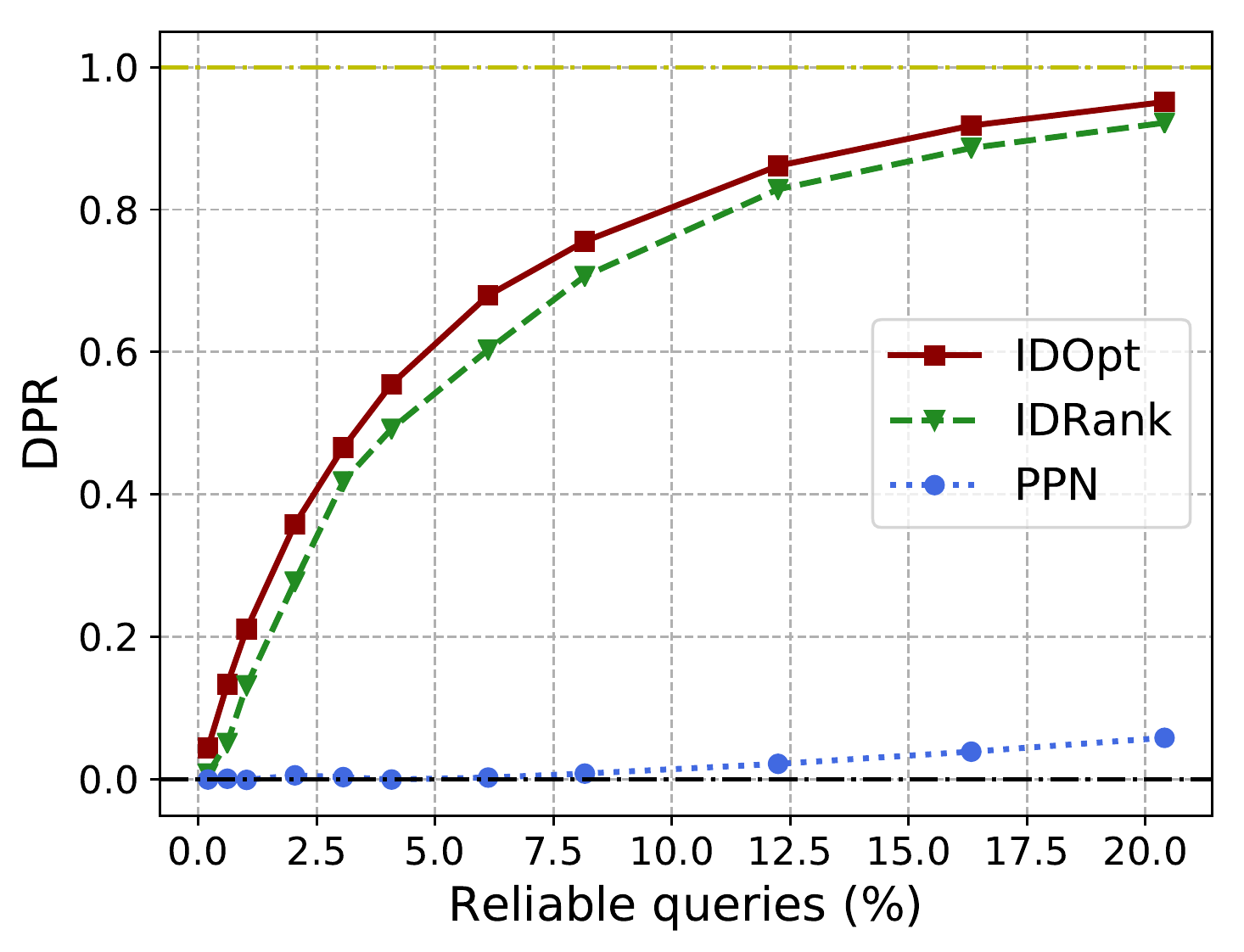}
		\caption{S\o rensen}
	\end{subfigure}	
	
	\begin{subfigure}[t]{0.235\textwidth}
		\centering
		\includegraphics[width=\textwidth,height = 2.5cm]{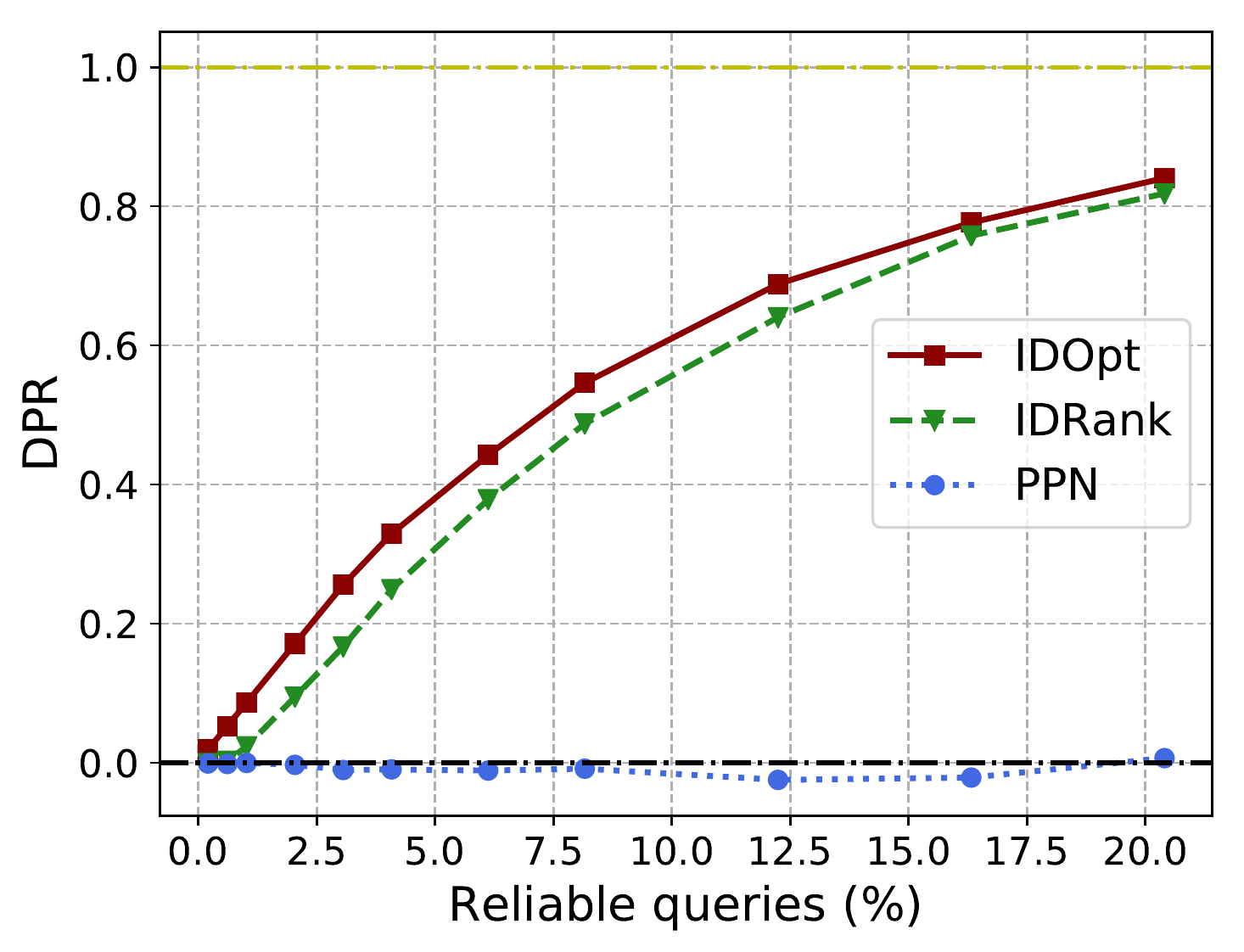}
		\caption{RA}
	\end{subfigure}
	\hfill
	\begin{subfigure}[t]{0.235\textwidth}
		\centering
		\includegraphics[width=\textwidth,height = 2.5cm]{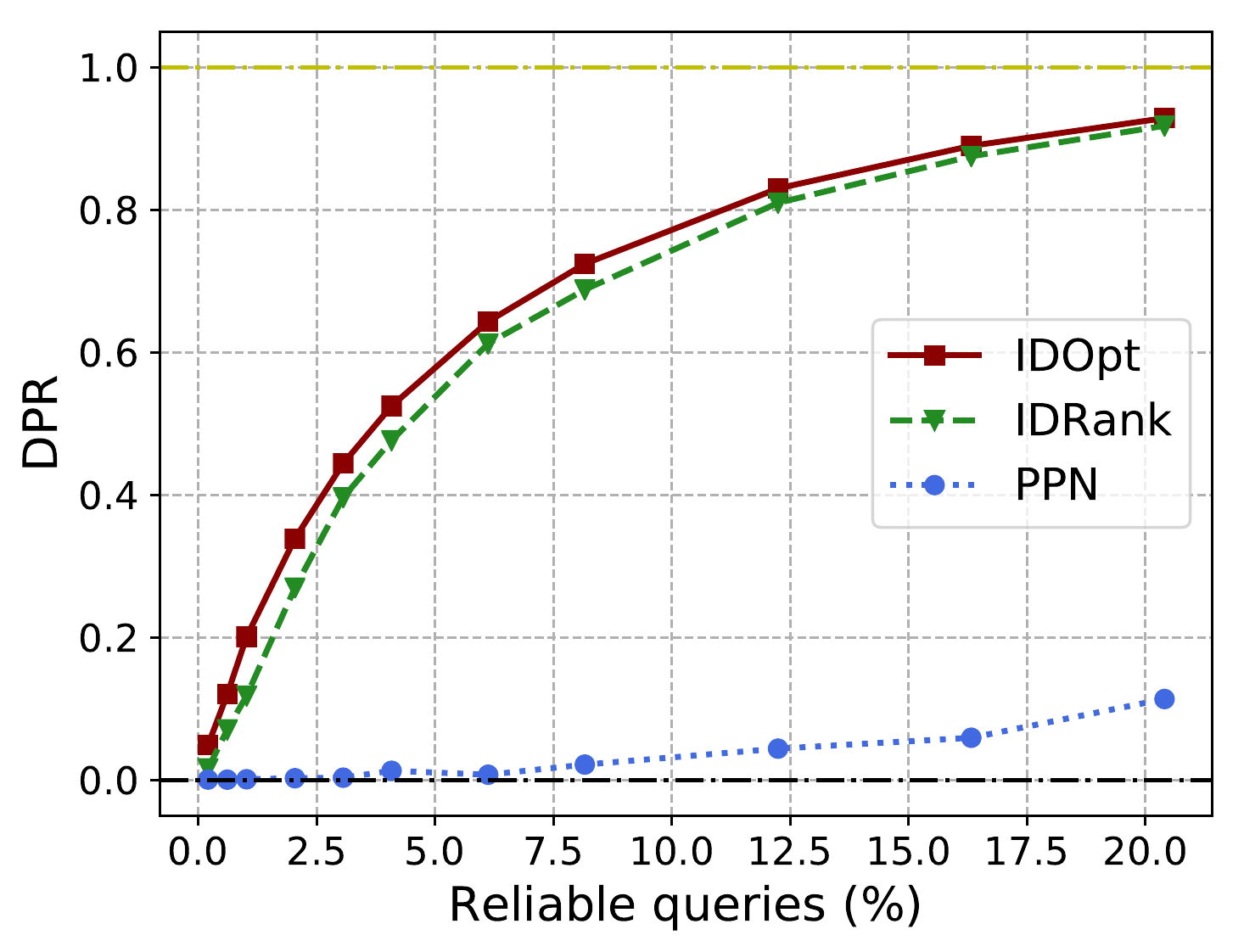}
		\caption{Salton}
	\end{subfigure}	
	\caption{ $\mathsf{DPR}$ under \textit{LinkDel} attack on Gov.}
	\label{fig-Gov}
\end{figure}

In general, by making a small portion of queries reliable, our
proposed \textit{IDOpt} and \textit{IDRank} can alleviate much of the damage caused by
attacks, with \textit{IDRank} nearly as good as \textit{IDOpt} in most
cases. For example, for the CN metric on PA dataset, making $100$ reliable
queries ($\sim 2\%$) can prevent almost $60\%$ of the damage and $500$
reliable queries ($\sim 10\%$) can prevent around $80\%$ of the
damage if we use our proposed approach, whereas \textit{PPN} has
virtually no effect even with $20\%$ reliable queries.
Throughout, we observe diminishing returns to investment in reliable
queries, an observation that is most evident on the PLD dataset.

\paragraph{Defense under \textit{UnbiasDel} and \textit{RandDel} attacks}
We further test the defense performance under the \textit{UnbiasDel}
and \textit{RandDel} attacks, respectively, under the same experiment
settings. 
The $\mathsf{DPR}$ on the PA and TVShow datasets are presented in
Fig.~\ref{fig-PA-both} and Fig.~\ref{fig-TVShow-both},
respectively. The results show that \textit{IDOpt} and \textit{IDRank}
can successfully limit the damage from these two additional attacks,
even though they were not explicitly designed with these in mind.

\begin{figure}[htp!]
	\centering
	\begin{subfigure}[t]{0.235\textwidth}
		\centering
		\includegraphics[width=\textwidth,height = 2.5cm]{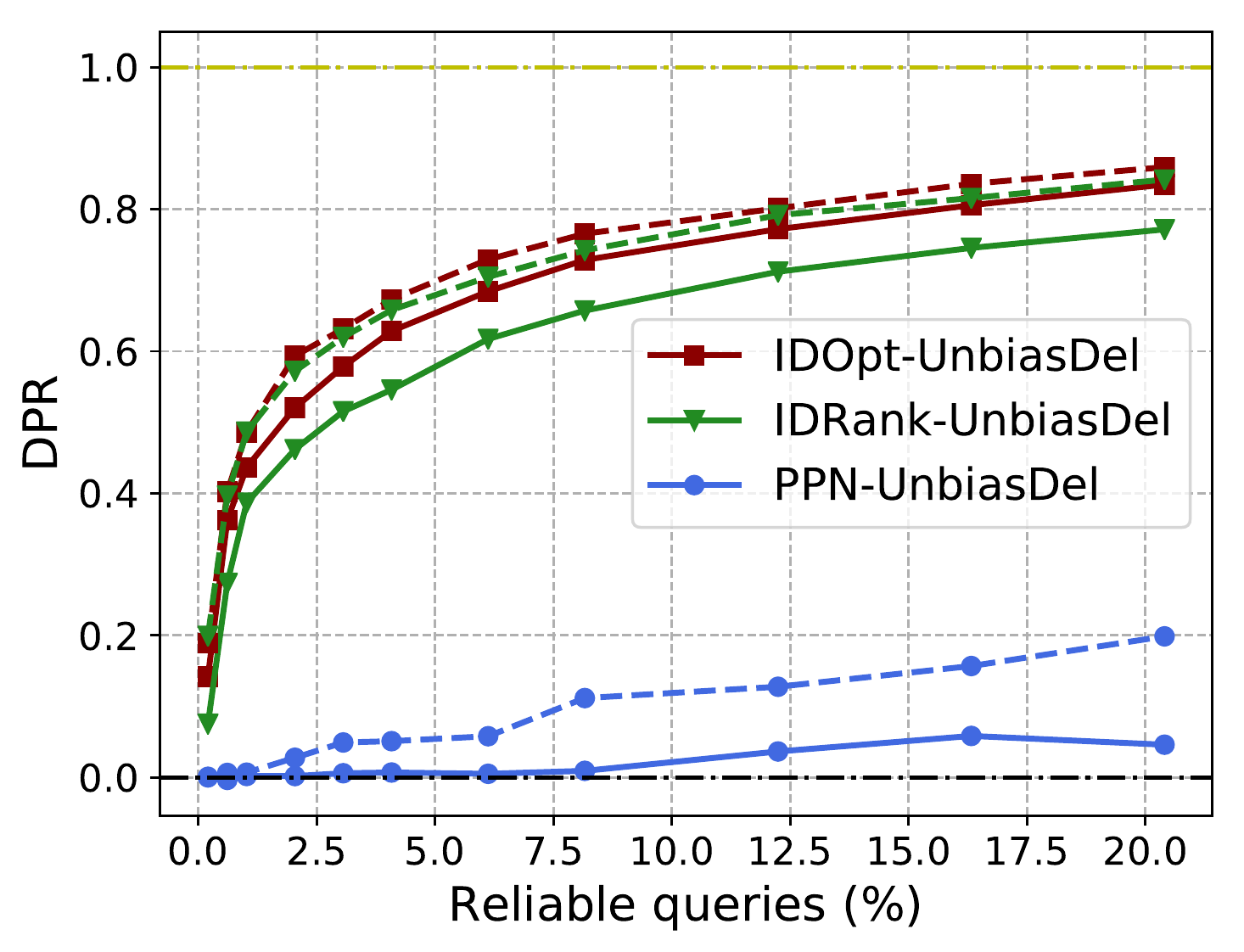}
		\caption{CN}
	\end{subfigure}%
	\hfill
	\begin{subfigure}[t]{0.235\textwidth}
		\centering
		\includegraphics[width=\textwidth,height = 2.5cm]{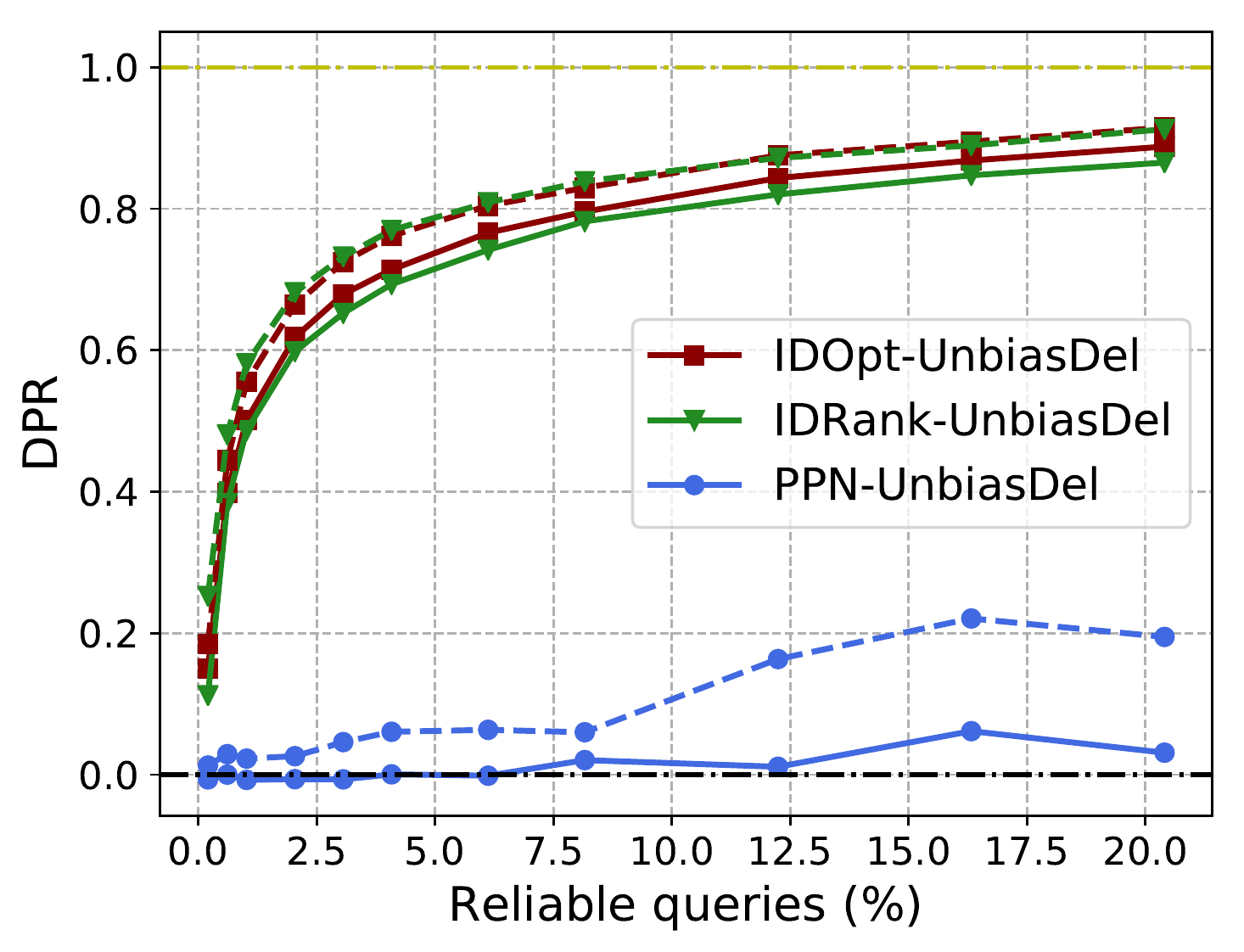}
		\caption{S\o rensen}
	\end{subfigure}	
	
	\begin{subfigure}[t]{0.235\textwidth}
		\centering
		\includegraphics[width=\textwidth,height = 2.5cm]{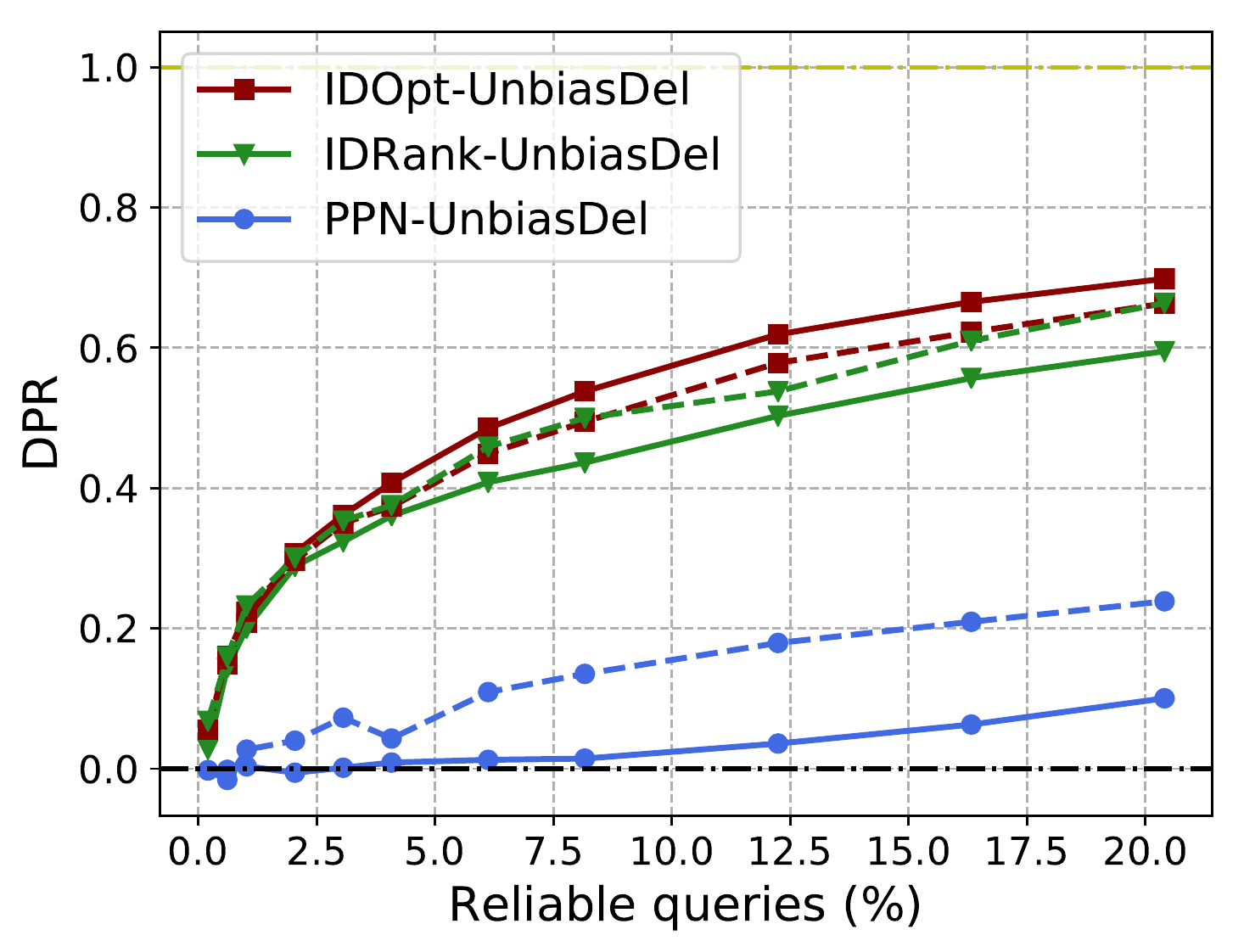}
		\caption{RA}
	\end{subfigure}
	\hfill
	\begin{subfigure}[t]{0.235\textwidth}
		\centering
		\includegraphics[width=\textwidth,height = 2.5cm]{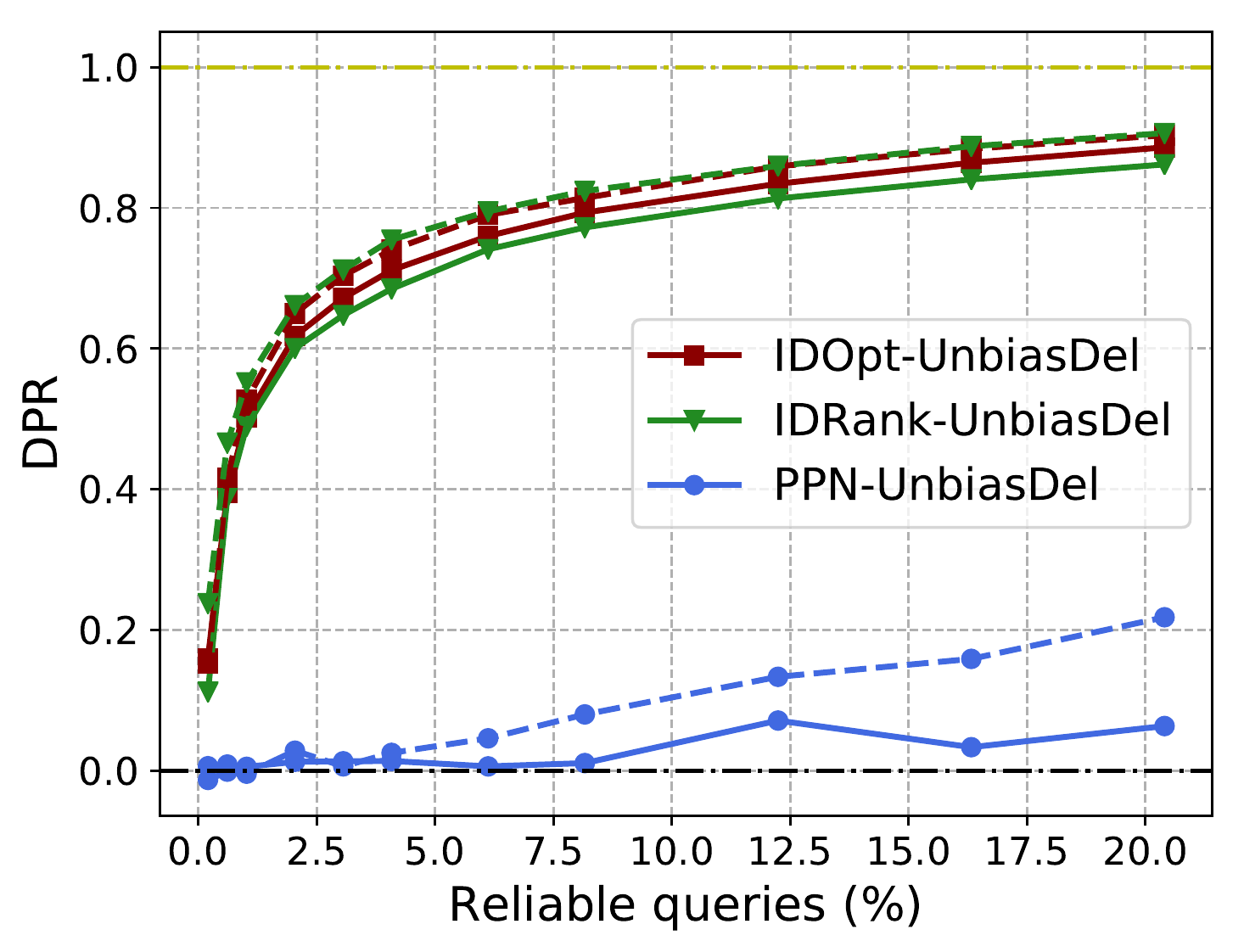}
		\caption{Salton}
	\end{subfigure}	
	\caption{$\mathsf{DPR}$  under \textit{UnbiasDel} attack (solid lines) and \textit{RandDel} attack (dotted lines with the corresponding color) on PA.}
	\label{fig-PA-both}
\end{figure}

\begin{figure}[htp!]
	\centering
	\begin{subfigure}[t]{0.235\textwidth}
		\centering
		\includegraphics[width=\textwidth,height = 2.5cm]{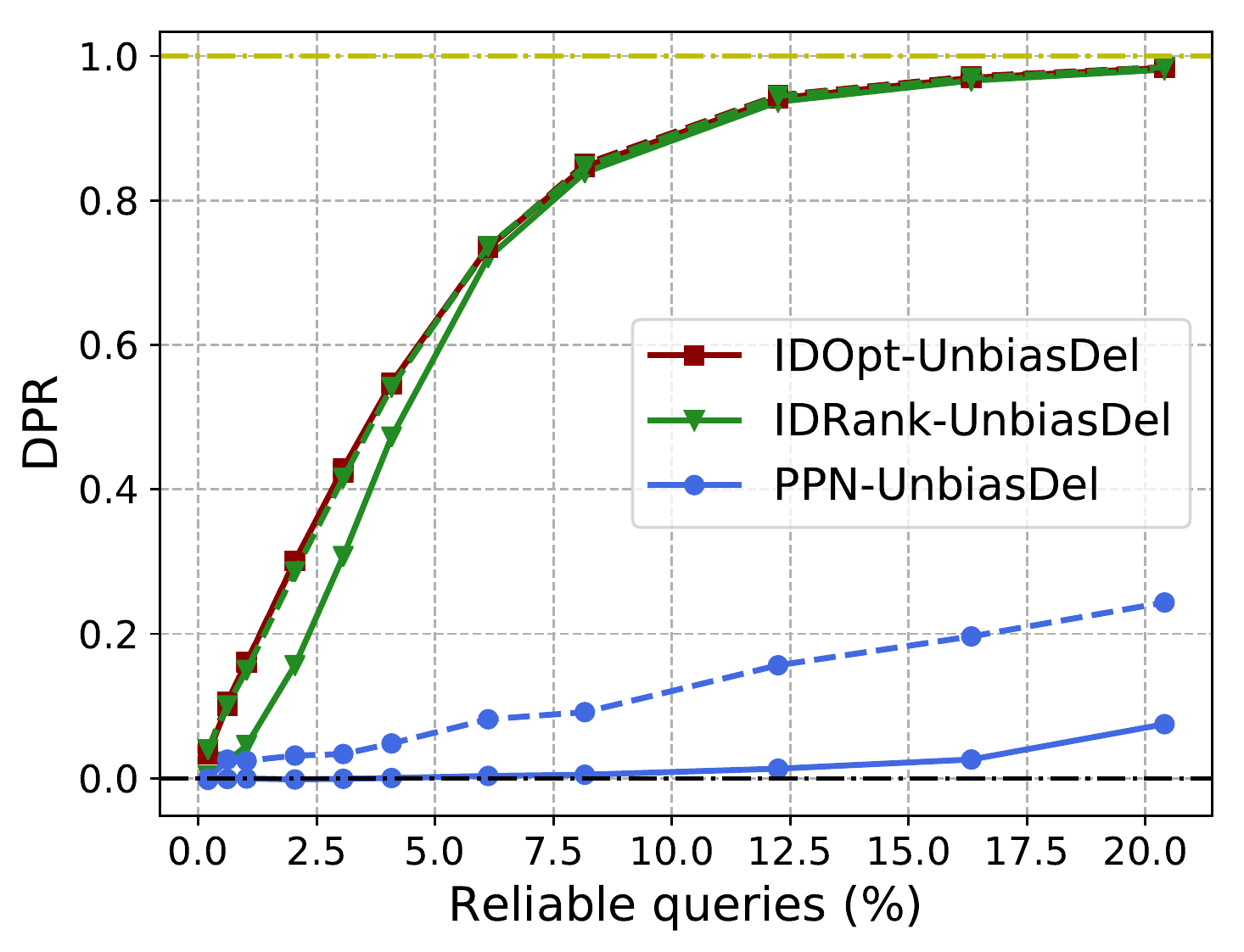}
		\caption{CN}
	\end{subfigure}%
	\hfill
	\begin{subfigure}[t]{0.235\textwidth}
		\centering
		\includegraphics[width=\textwidth,height = 2.5cm]{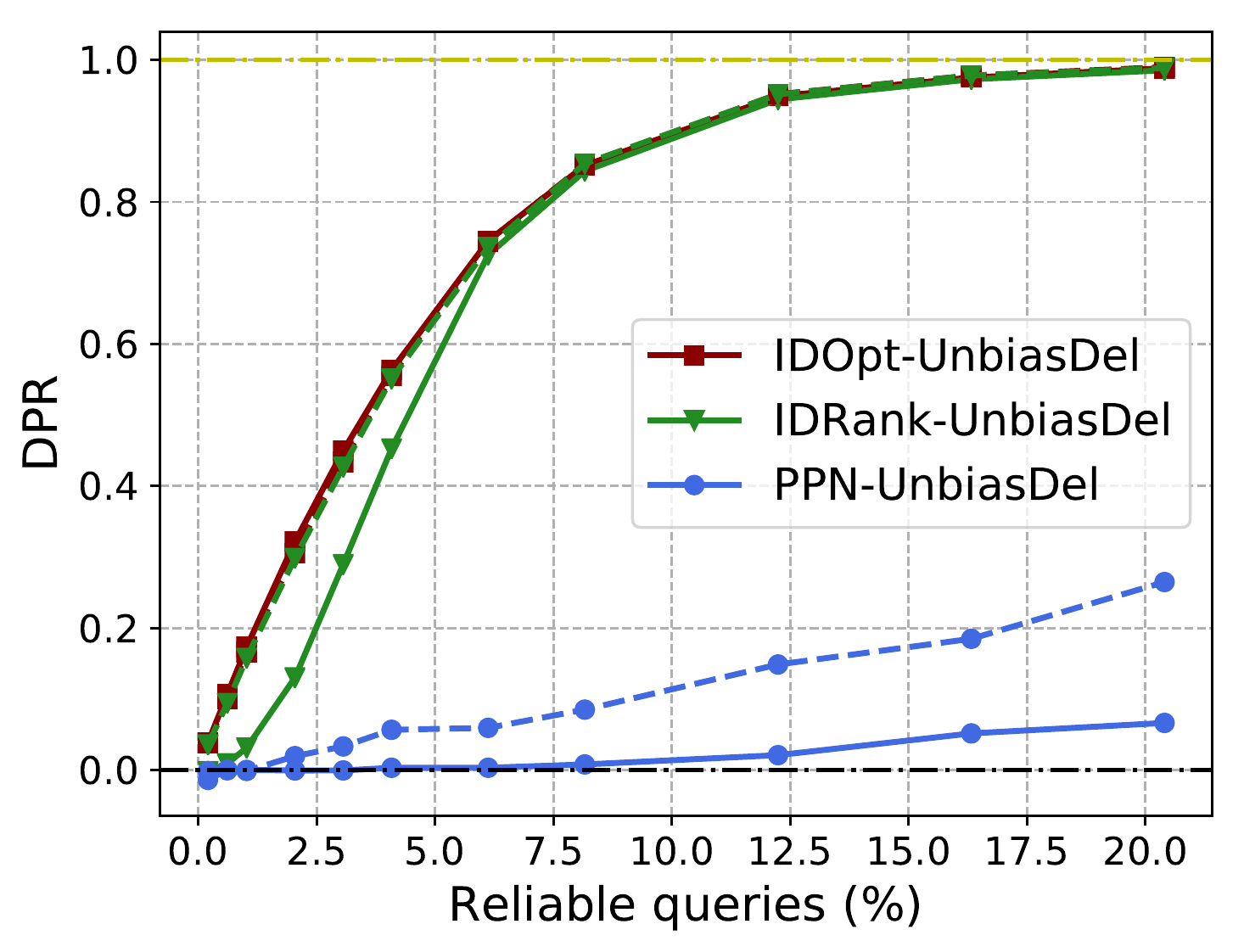}
		\caption{S\o rensen}
	\end{subfigure}	
	
	\begin{subfigure}[t]{0.235\textwidth}
		\centering
		\includegraphics[width=\textwidth,height = 2.5cm]{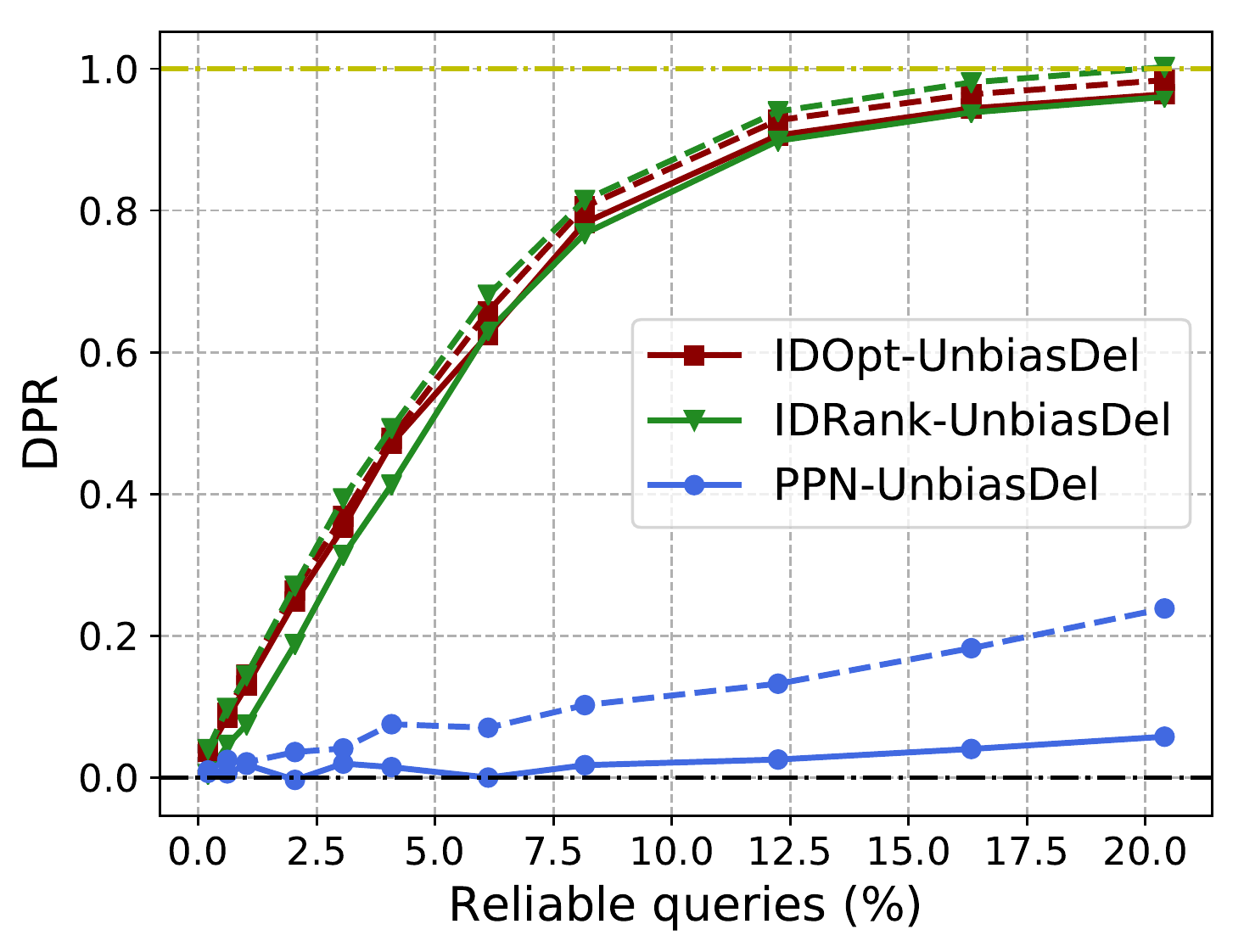}
		\caption{RA}
	\end{subfigure}
	\hfill
	\begin{subfigure}[t]{0.235\textwidth}
		\centering
		\includegraphics[width=\textwidth,height = 2.5cm]{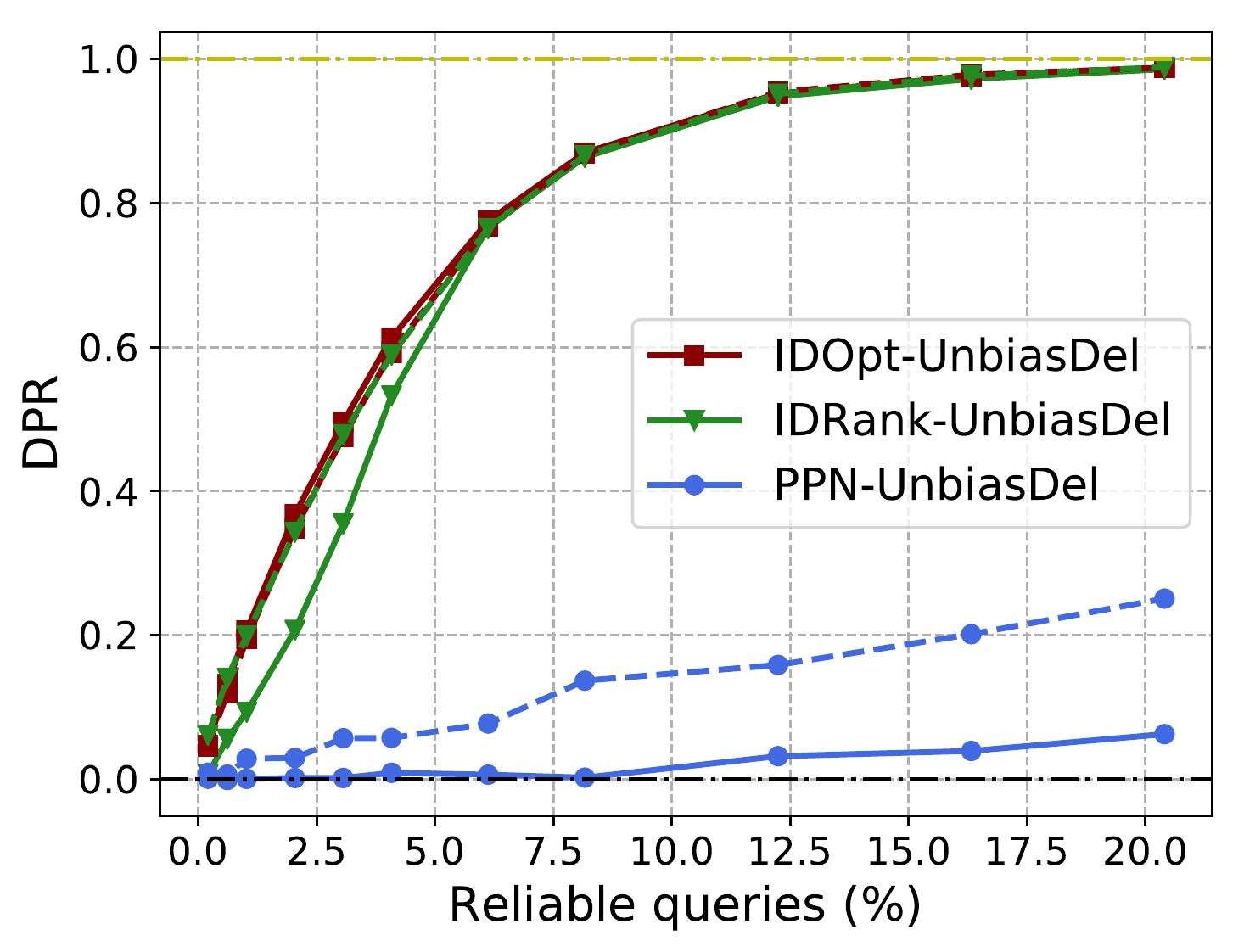}
		\caption{Salton}
	\end{subfigure}	
	\caption{$\mathsf{DPR}$  under \textit{UnbiasDel} attack (solid lines) and \textit{RandDel} attack (dotted lines with the corresponding color) on TVShow.}
	\label{fig-TVShow-both}
\end{figure}

%---------------  Conclusion -----------
\section{Conclusion}
In this paper, we initiate the study of making link prediction robust
against attacks that remove edges from the observed network in order
to hide a target link.
Specifically, we endow the analyst with the ability to make a small
set of reliable queries that return accurate results about associated
edges and cannot be manipulated by the adversary.
We then model the problem of robust link prediction in this context as
a Bayesian Stackelberg game in which the defender chooses which
reliable queries to make, and the attacker then deletes a subset of
links.
In this game, the analyst (defender) is uncertain about both the true
graph structure and the attacker's preference about which link to
hide.
We show that solving this game is NP-Hard, but also present two
approaches for approximately solving it. Our extensive experimental evaluation demonstrates that our robust link prediction approach is quite effective in defending against several attacks.

\section*{Acknowledgment}
This work was partially supported by the NSF (IIS-1905558) and ARO (W911NF1610069, MURI W911NF1810208). Tomasz P. Michalak was supported by the Polish National Science Centre grant 2016/23/B/ST6/03599.

\bibliographystyle{plain}
\bibliography{citation}

\end{document}